\documentclass{article}

\usepackage[english]{babel}
\usepackage[utf8x]{inputenc}
\usepackage[T1]{fontenc}
\usepackage[a4paper,top=3cm,bottom=2cm,left=3cm,right=3cm,marginparwidth=1.75cm]{geometry}

\usepackage{comment}
\usepackage{times}
\usepackage{amsfonts}       
\usepackage{amsmath}
\usepackage{amsthm}
\usepackage{multicol}
\usepackage{subcaption}
\usepackage{adjustbox}
\usepackage{enumitem}
\usepackage{lipsum}
\usepackage{algorithm}
\usepackage{wrapfig}
\usepackage[noend]{algpseudocode}
\setlist{nolistsep}

\newtheorem{theorem}{Theorem}
\newtheorem{corollary}{Corollary}
\newtheorem{proposition}{Proposition}

\makeatletter
\newcommand{\printfnsymbol}[1]{%
  \textsuperscript{\@fnsymbol{#1}}%
}
\makeatother

\usepackage{graphicx}
\title{Importance Weighted Adversarial Variational Autoencoders for Spike Inference from Calcium Imaging Data}

\author{%
  Daniel Jiwoong Im, Sridhama Prakhya\thanks{equal contribution}, Jinyao Yan\printfnsymbol{1}, Srinivas Turaga, Kristin Branson\\ 
  Janelia Research Campus, HHMI, Virginia\\ 
  \texttt{imd@janelia.hhmi.org} \\
}

\begin{document}

\maketitle

\begin{abstract}
The Importance Weighted Auto Encoder (IWAE) objective has been shown to improve the training of generative models over the standard Variational Auto Encoder (VAE) objective. 
    Here, we derive importance weighted extensions to Adversarial Variational Bayes (AVB) and Adversarial Autoencoder (AAE). These latent variable models use implicitly 
    defined inference networks whose approximate posterior density $q_\phi(z|x)$ cannot be directly evaluated, an essential ingredient for importance weighting. We show 
    improved training and inference in latent variable models with our adversarially trained importance weighting method, and derive new theoretical connections between 
    adversarial generative model training criteria and marginal likelihood based methods. We apply these methods to the important problem of inferring spiking neural activity 
    from calcium imaging data, a challenging posterior inference problem in neuroscience, and show that posterior samples from the adversarial methods outperform factorized 
    posteriors used in VAEs.
\end{abstract}

\section{Introduction}


The variational autoencoder (VAE) \cite{Kingma2014vae,Rezende2014} has been used to train deep latent variable based generative models which model a distribution over observations $p(x)$ by latent variables $z$ such that $p(x)=\int dz p_\theta(x|z).p(z)$ using a deep neural network $\theta$ which transforms samples from $p(z)$ into samples from $p(x)$. This model trains the latent variable based generative model using approximate posterior samples from a simultaneously trained recognition network or inference network $q_\phi(z|x)$ to maximize the evidence lower bound (ELBO).

There are two ways to improve the quality of the learned deep generative model. The multi-sample objective used by the importance weighted autoencoder (IWAE) \cite{Burda2015} has been used to derive a tighter lower bound to the model evidence $p(x)$, leading to superior generative models. Optimizing this objective corresponds to implicitly reweighting the samples from the approximate posterior. A second way to improve the quality of the generative model is to explicitly improve the approximate posterior samples generated by the recognition network.

In the VAE framework, the recognition network is restricted to approximate posterior distributions under which the log probability of a sample and its derivatives can be evaluated in close form. The adversarial autoencoder (AAE) \cite{Makhzani2016}, and the adversarial variational Bayes (AVB) \cite{Mescheder2017} show how this constraint can be relaxed, leading to more flexible posterior distributions which are implicitly represented by the recognition network. In this paper, we derive importance weighted adversarial autoencoders of IW-AVB and IW-AAE, thus combining both adversarial and importance weighting techniques for improving probabilistic modeling.

Spike inference is an important Bayesian inference problem in neuroscience \cite{berens2018community}. Calcium imaging methods enable the indirect measurement of neural activity of large populations of neurons in the living brain in a minimally invasive manner. The intracellular calcium concentration measured by fluorescence microscopy of a genetically encoded calcium sensor such as GCaMP6 \cite{chen2013ultrasensitive} is an indirect measure of the spiking activity of the neuron. VAEs have previously been used \cite{speiser2017fast,Aitchison:2017wz} to perform Bayesian inference of spiking activity by training inference networks to invert the known biophysically described generative process which converts unobserved spikes into observed fluorescence time series.

The accuracy of a VAE-based spike inference method depends strongly on the quality of the posterior approximation used by the inference network $q_\phi$. The posterior distribution over the binary latent spike train $s =\{s_1,...,s_T\}$ given the fluorescence time series $f=\{f_1,...,f_T\}$ has previously been approximated \cite{speiser2017fast} using either a factorized Bernoulli distribution (VIMCO-FAC) where $p(s|f)\approx\Pi_{i} q_\phi(s_i|f)$, or as an autoregressive Bernoulli distribution $p(s|f)\approx\Pi_{i} q_\phi(s_i|f,s_1,...,s_{i-1})$ (VIMCO-CORR). As we show, the correlated autoregressive posterior is more accurate, but slow to sample from. In contrast, the factorized posterior allows for fast parallel sampling, especially on a GPU, but ignores correlations in the posterior. Fast inference networks which sample from correlated posteriors over discrete binary spike trains would be a significant advance for VAE-based spike inference.

Fast correlated distributions over time series can be constructed using normalizing flows for continuous random variables \cite{Rezende2015}, but this is considerably harder for discrete random variables \cite{Aitchison:2018vq}. Thus an adversarial approach where an inference network which transforms noise samples into samples from the posterior can be trained without the need to evaluate the posterior likelihood $q(s|f)$ is particularly appealing for modeling correlated distributions over discrete random variables. Here, we show that our adversarially trained inference networks produce correlated samples which outperform the factorized posterior trained in the conventional way as in \cite{speiser2017fast}.

In addition to these practical advances, we derive theoretical results connecting the objective functions optimized by the importance weighted variants of the AVB, AAE, and VAE.
The relationship between the AAE objective and data log likelihood is not fully understood. The AAE has been shown to be a special case of the Wasserstein autoencoder (WAE) under certain restricted conditions \cite{Bousquet2017}. However, we also do not understand the tradeoffs between the standard log-likelihood and penalized optimal transport objectives, and thus further theoretical insight is necessary to fully understand the tradeoffs between the VAE and AAE. 

The main contributions of the paper are following:

\begin{enumerate}
    \item We propose IW-AVB and IW-AAE that yield tighter lower bounds on log-likelihood compared to AVB and AAE, and the global solution for maximizes likelihood.
    
    \item We provide theoretical insights into the importance weighted adversarial objective functions. 
    In particular, we relate AAE and IW-AAE objectives to log-likelihoods and Wasserstein autoencoder objectives.
    
    
    \item We develop standard and importance weighted adversarial neural spike inference for calcium imaging data, and show that adversarially trained inference networks outperform existing VAEs using factorized posteriors.
    
\end{enumerate}

\section{Background}

The maximum likelihood estimation of the parameter $\theta$ with
model defined as $p_\theta(x) = \int p_\theta(x|z)p(z)dz$,
where $z$ is a latent variable is in general intractable.
Variational methods maximize a lower bound of the log likelihood. This lower bound is based on approximating the intractable distribution $p(z|x)$ by a tractable
distribution $q_\phi(z|x) \in \mathcal{Q}$ parameterized by variational parameter $\phi$. VAEs maximize the following lower bound of $\log p_\theta(x)$:
$\mathcal{L} = \mathbb{E}_{z\sim q_\phi(\bf{z}|\bf{x})}
            \left[ \log
            \frac{p_\theta(x,z)}{q_\phi(z|x)}\right].$
To make the relationship with proposed methods clear, we write this as
\begin{align}
    \mathcal{L}_{\text{VAE}} := \mathbb{E}_{z_1,\cdots, z_k\sim q_\phi(z|x)}
            \left[ \frac{1}{k}\sum^{k}_{i=1} \log 
            \frac{p_\theta(x,z_i)}{q_\phi(z_i|x)}\right].
\end{align}
We do this for all criteria going forward. 

To efficiently optimize this criterion with gradient descent, VAEs~\cite{Kingma2014vae,Rezende2014} 
define the approximate posterior $q_\phi(z|x)$ such that the $z$ is a differentiable
transformation $g_\phi(\epsilon,x)$ of an noise variable $\epsilon$.
It is common to assume $z = g_\phi^\mu(x) + diag(g_\phi^\sigma(x)) \epsilon$ and $\epsilon \sim \mathcal{N}(0,I)$, and for $g_\phi$ to be a deep network with weights $\phi$.

Requiring that $\log q_\phi$ can be analytically evaluated restricts the class $\mathcal{Q}$ and is a limitation to such approaches. Adversarial variational Bayes (AVB) \cite{Mescheder2017} maximizes the variational lower
bound by implicitly approximating KL divergence between 
approximate posterior $q_\phi(z|x)$ and the prior distribution $p(z)$
by introducing third neural network, $T_\psi$.
This neural network, known as the discriminator, implicitly estimates 
$\log p(z) - \log q_{\phi}(z|x)$.
\begin{align}
\mathcal{L}_{\text{AVB}} :=&
       \mathbb{E}_{z_1,\cdots, z_k\sim q_\phi}
            \bigg[ \frac{1}{k}\sum^{k}_{i=1} \bigg(\log 
            p_\theta(x|z_i)
            - T^*(x,z_i)\bigg)\bigg]\\
    T^*(x,z) =& \log q_\phi(z|x) - \log p(z) \nonumber \\
    \approx& \max_{\psi}
        \mathbb{E}_{x\sim {p_{\mathcal{D}}}}\left[
            \mathbb{E}_{q_\phi(z|x)} \bigg[ \log \sigma (T_\psi(x,z)) \right]
            + \mathbb{E}_{p(z)} \log (1- \sigma(T_\psi(x,z)))\bigg] 
    \label{eqnTstar_xz}
\end{align}
The three parametric models $p_\theta$, $q_\phi$ and $T_\psi$ are
jointly optimized using adversarial training.
Unlike VAE and IWAE, in this framework, we can make arbitrarily flexible approximate distributions $q_\phi$.

The adversarial autoencoder (AAE) \cite{Makhzani2016} is similar, except that the discriminative network $T_\psi$ depends only on $z$, instead of on $x$ and $z$. AAE objective minimizes the following objective:
\begin{align}
    \mathcal{L}_{\text{AAE}} :=&
        \mathbb{E}_{z_1,\cdots, z_k\sim q_\phi(\bf{z}|\bf{x})}
            \left[ \frac{1}{k}\sum^{k}_{i=1} \left(\log 
            p_\theta(x|z_i) - T^*(z_i)\right)\right] \\
    T^*(z) =& \log \int q_\phi(z|x)p(x)dx - \log p(z) \nonumber \\
    \approx& \max_{T_\psi}
        \mathbb{E}_{p_{\mathcal{D}}(x)}\left[
            \mathbb{E}_{q_\phi(z|x)} \bigg[ \log \sigma (T_\psi(z)) \right]\bigg]  + 
        \mathbb{E}_{p(z)} \log (1- \sigma(T_\psi(z))) 
    \label{eqnTstar_z}
\end{align}

AAE replaces the KL divergence between the approximate posterior
and prior distribution in $\mathcal{L}_{\text{VAE}}$ with an adversarial loss
that tries to minimize the divergence between the aggregated posterior 
$\int q_\phi (z|x)p_\mathcal{D}(x)dx$ and the prior distribution $p(z)$.

\section{Importance Weighted Adversarial Training}
The importance weighted autoencoder (IWAE) \cite{Burda2015} provides a tighter 
lower bound to $\log p_\theta(x)$,
\begin{align} 
       \mathcal{L}_{\text{IWAE}} := \mathbb{E}_{z_1,\cdots, z_k\sim q_\phi(\bf{z}|\bf{x})}
            \left[ \log \left(\frac{1}{k}\sum^{k}_{i=1} 
            \frac{p_\theta(x,z_i)}{q_\phi(z_i|x)}\right)\right]
            \label{eqn:iwae}.
\end{align}
{\em Burda et al.} \cite{Burda2015} show that 
$\mathcal{L}_{\text{VAE}}=\mathcal{L}_{\text{IWAE},k=1} \leq \mathcal{L}_{\text{IWAE},k=2} \leq \ldots \leq \log p(x)$,
and $\mathcal{L}_{\text{IWAE}}$ approaches $\log p_\theta(x)$ as $k \rightarrow \infty$.

\subsection{IW-AVB and IW-AAE}
    In AVB, generative adversarial training on
        joint distributions between data and latent variables 
        is applied to the variational lower bound $\mathcal{L}_{AVB}$. 
        In this work, we propose applying it to the importance weighted lower bound of 
        $\log p(x)$,
        \begin{align}
               \mathcal{L}_{\text{IW-AVB}} := \mathbb{E}_{z_1,\cdots, z_k\sim q_\phi(\bf{z}|\bf{x})}
                    \left[ \log \frac{1}{k}\sum^{k}_{i=1} p(x|z_i) 
                    \left(\exp (-T^*(x,z_i))\right)  \right]
                     \label{eqn:aiwbo}
        \end{align}
        where $T^*(x,z)$ is defined as in Equation~\ref{eqnTstar_xz}.
        We call this Importance Weighted Adversarial Variational Bayes bound (IW-AVB).
        As $T^*(x,z)= - \log \frac{p(z)}{q_\phi(z|x)}$, 
        $\mathcal{L}_{\text{IW-AVB}}$ as $\mathcal{L}_{\text{IWAE}}$.

        The main advantage of IW-AVB over AVB is that, when the true 
        posterior distribution is not in the class of approximate posterior functions (as is generally the case), IW-AVB uses a tighter lower bound 
        than AVB \cite{Burda2015}.

    Similarly, we can apply importance weighting to improve AAE:
        \begin{align} 
            \mathcal{L}_{\text{IW-AAE}} := 
                \mathbb{E}_{z_1,\ldots, z_k\sim q_\phi(\bf{z}|\bf{x})}
                    \left[ \log 
                    \frac{1}{k}\sum^{k}_{i=1} \left(p(x|z_i) \exp (-T^*(z_i))\right)\right] \label{eqn:IW-AAE_org}
        \end{align}
    where $T^*(z)$ is defined as in Equation~\ref{eqnTstar_z}.

IW-AVB and IW-AAE objectives can be described as a framework of 
minimax adversarial game between three neural networks, the
generative network $p_\theta(x|z)$, inference network $q_\phi(z|x)$,
and discriminative network $T_\psi(x,z)$.
The inference network maps input $x$ to latent space $Z$, and
the generative network maps latent samples to the the data space $X$.
Both inference and generative networks are jointly trained to minimize the
reconstruction error and KL divergence term in $D_{\text{KL}}(q_\phi(z|x)\|p(z))$.
The discriminator network $T_\psi(x,z)$ differentiates samples from the joint distribution between 
data and approximate posterior distribution (positive samples)
versus the samples that are from the joint over data and prior 
latent distribution (negative samples).


    Recent work \cite{Rainforth2018} has shown that optimizing the importance weighted bound can  
        degrade the overall learning process of the inference network because the 
        signal-to-noise ratio of the gradient estimates  
        $SNR_{k}(\eta) = \left| \frac{\mathbb{E}[\Delta_{k}(\eta)}{\sigma[\Delta_{k}(\eta)]} \right|$
        converges at the rate of $O(\sqrt{k})$ and $O(\sqrt{1/k})$ 
        for generative and inference networks, respectively
        ($\Delta_{k}(\eta) = \nabla_{\eta} \log \frac{1}{k}\sum^{k}_{i=1} \frac{p_\theta(x,z_i)}{q_\phi(z_i|x)}$ is the gradient estimate of $\eta$).
        The $SNR_{k}(\phi)$ converges to 0 for inference network as $k \rightarrow \infty$, and the gradient estimates of $\phi$ become completely random.
        To mediate this, we apply the importance weighted bound for updating 
        the parameter of generative network $p_\theta(x|z)$ and 
        variational lower bound for updating the parameters of inference network $q_\phi(z|x)$.
        Hence, we maximize the following: 
        \begin{align}
            &\max_{\theta} \mathbb{E}_{z_1,\cdots, z_k\sim q_\phi(\bf{z}|\bf{x})}
                    \left[ \log \frac{1}{k}\sum^{k}_{i=1} p_\theta(x|z_i) \exp \left(-T^*(x,z_i)\right)\right]\nonumber\\
            &\max_{\phi} \mathbb{E}_{z_1,\cdots, z_k\sim q_\phi(\bf{z}|\bf{x})}
                    \left[ \frac{1}{k}\sum^{k}_{i=1} \log p_\theta(x|z_i) - T^*(x,z_i) \right].
            \label{eqn:aiwb2b}
        \end{align}
    We do this for IW-AAE as well.
    
We alternate between updating inference-generative pair, and adversarial 
discriminator $T_\psi$. The training procedures for IW-AVB and IW-AAE 
are shown in Algorithm~\ref{algo:IW-AVB} and \ref{algo:IW-AAE}. 
    
\begin{figure}[t]
\vspace{-0.7cm}
    \begin{algorithm}[H]
    \caption{IW-AVB}
    \label{algo:IW-AVB}
    \begin{algorithmic}[1]
        \State Initialize $\theta$, $\phi$, and $\psi$. 

        \While {$\phi$ has not converged}
            \State /\ /\ $K$ latent samples for $N$ data points.
            \State Sample $\lbrace x^{(1)}, \ldots, x^{(N)} \rbrace \sim p_{\mathcal{D}}(x)$. 
            \State Sample $\lbrace \epsilon^{(1)}, \ldots, \epsilon^{(NK)} \rbrace \sim \mathcal{N}(0,1)$.
            \State Sample $\lbrace z^{(1)}, \ldots, z^{(NK)} \rbrace \sim p(z)$.
            \State Sample $\lbrace \tilde{z}^{(k|n)} \rbrace \sim q(z|x^{(n)},\epsilon^{(nk)})$ 

            \State /\ /\ Update the model parameters 
            \State Update $\theta \leftarrow \phi + h g_\psi$ where we compute the gradient of Eq.~\ref{eqn:aiwb2b}: 
            \State \qquad \qquad $g_\theta = \frac{1}{N}\sum^N_{n=1} \nabla_\theta \log \big[\frac{1}{K}\sum^{K}_{k=1} p_\theta \left(x^{(n)}|\tilde{z}^{(k|n)}\right)
                \exp\left(- T_{\psi}\left(x^{(n)}, \tilde{z}^{(k|n)}\right)\right)\big]$
            \State Update $\phi \leftarrow \phi + h g_\psi$ where we compute the gradient of Eq.~\ref{eqn:aiwb2b}: 
            \State \qquad \qquad $g_\phi = \frac{1}{N}\sum^N_{n=1} \nabla_\phi \big[\log p_\theta \left(x^{(n)}|\tilde{z}^{(1|n)}\right) - T_{\psi}\left(x^{(n)}, \tilde{z}^{(1|n)}\right)\big]$
            \State Update $\psi \leftarrow \psi + h g_\psi$  where we compute the gradient of Eq.~\ref{eqnTstar_xz}: 
            \State \qquad \qquad $g_\psi = \frac{1}{N}\sum^N_{n=1} \nabla_\psi \big[ 
                \log \left( \sigma(T_\psi(x^{(n)}, \tilde{z}^{(1|n)}))\right)+ \log \left( 1 -  \sigma(T_\psi(x^{(n)}, z^{(n)})) \right) \big]$
        \EndWhile
    \end{algorithmic}
    \end{algorithm}
\end{figure}

\subsection{Properties}
    
    An important reason to maximize $\phi$ w.r.t the variational lower bound in Equation~\ref{eqn:aiwb2b} is that it guarantees 
        $\mathbb{E}_{q_\phi(z|x)} \left[ \nabla_\phi T^*(x,z)\right] = 0$
        for the optimal discriminator network $T^*$ \cite{Mescheder2017}.
        Since deriving $T^*(x,z)$ indirectly depends on $\phi$, 
        we want the gradient w.r.t $\phi$  in $T^*(x,z)$ to be disentangled
        from calculating the gradients of Equation~\ref{eqn:aiwb2b}.
        Thus, we are only using the importance weighted bound on generative model. Empirically, we find that this still improves performance (Section~\ref{sec:experiments}).

        


The following proposition shows that the global Nash equilibria of 
IW-AVB's adversarial game yield global optima of the objective function
in $\mathcal{L}_{\text{IW-AVB}}$.

\begin{proposition}
    Assume $T$ can represent any function of two variables.
    If $(\theta^*, \phi^*, T^*)$ is a Nash Equilibrium of the two-player game for IW-AVB,
    then $-T^*(x,z) = \log \frac{p(z)}{q_{\phi^*}(z|x)}$
    and $(\theta^*, \phi^*)$ is a global optimum of the 
    importance weighted lower bound in Equation~\ref{eqn:aiwb2b}. 
    \label{prop1}
\end{proposition}
See the Appendix for proof.
This proposition tells us that the solution to Equation~\ref{eqn:aiwb2b}
gives the solution to importance weighted bound,
in which $\theta^*$ becomes the maximum likelihood assignment.

A similar property holds for AAE and IW-AAE with the discriminator $T_\psi(z)$.
\begin{proposition}
    Assume $T$ can represent any function of two variables.
    If $(\theta^*, \phi^*, T^*)$ is a Nash Equilibrium of two-player game for IW-AAE,
    then $-T^*(z) = \log \frac{p(z)}{q_{\phi^*}(z)}$
    and $(\theta^*, \phi^*)$ is the global optimum of the following objective, 
    \begin{align}
        \mathbb{E}_{z_1,\ldots, z_k \sim q^{IW}_\phi(z|x)} \left[ \frac{1}{k} \sum^{k}_{i=0} \log p(x|z_i)\frac{p(z_i)}{q^{IW}_\phi(z_i)}\right]
       \label{eqn:aae_primal}
    \end{align}
    where $q^{IW}_\phi(z_i|z_{\/i}) := \frac{p(x,z_i)}{\frac{1}{k}\sum^{k}_{j=1}\frac{p(x,z_j)}{q(z_j)}}$.
    \label{prop2}
\end{proposition}
The steps of the proof are the same as for Proposition~\ref{prop1}.

In the next section, we provide theoretical insights into the relationship between the optima of Equations~\ref{eqn:aiwb2b} and \ref{eqn:aae_primal} and the log-likelihood. 


\section{Analysis}

    Bousquet {\em et al.} \cite{Bousquet2017} showed adversarial objectives with equivalent solutions to $\mathcal{L}_{\text{AVB}}$ and $\mathcal{L}_{\text{AAE}}$.
    In a similar manner, we show that the adversarial objective with equivalent solutions to $\mathcal{L}_{\text{IW-AVB}}$ is
    {\small
        \begin{multline}
            \mathcal{D}_{\text{IW-AVB}}(p_{\mathcal{D}}, p_\theta) := 
            \min_{q_\phi(z|x)\in \mathcal{Q}} \Bigg[
                \mathbb{E}_{p_\mathcal{\mathcal{D}}(x)}\bigg[ D_{\text{GAN}}\left( q_\phi(z|x)\| p(z)\right) - 
                    \mathbb{E}_{z_1,\ldots, z_n \sim q_{\phi}(z|x)} 
                    \Bigg[ \frac{1}{k} \sum^k_{i=1} \log p(x|z_i)\bigg]\bigg]\Bigg]\\ - \max_{\theta \in \Theta} 
                    \mathbb{E}_{p_{\mathcal{D}}(x)}
                    \mathbb{E}_{z_1,\ldots, z_k\sim q_\phi(z|x)}
                    \bigg[\log \frac{1}{k}\sum^k_{j=0} p_\theta (x|z_i)
                    \exp(-T_\psi(x,z_j))\bigg]
                \label{eqn:aiwb3}
        \end{multline}}
        where $D_{\text{GAN}}$ is the generative adversarial network objective 
        \cite{Goodfellow2014} with discriminative network $T_\psi$, and $p_{D}$ and $p_\theta$ are 
        data and model distributions. 
        $D_{\text{IW-AVB}}(p_{\mathcal{D}}, p_\theta)$ can be viewed as (pseudo-) divergence 
        between the data and model distribution, 
        where $p_\theta(x) = \int p_\theta (x|z) p(z) dz$ for all $x$.
        
    Similarly, the the adversarial objective for IW-AAE becomes 
        {\small \begin{multline}
            D_{\text{IW-AAE}}(p_{\mathcal{D}}, p_\theta) 
                :=  \min_{q_\phi(z|x)\in \mathcal{Q}} \Bigg[
                D_{\text{GAN}}\left[q_\phi(z), p(z)\right]-
                \mathbb{E}_{p_\mathcal{D}(x)}
                    \mathbb{E}_{z_1,\ldots, z_k \sim q_{\phi}(z|x)} 
                    \bigg[\frac{1}{k} \sum^k_{i=1} \log  p(x|z_i)\bigg]\Bigg]\\
                    - \max_{\theta\in\Theta}
                \mathbb{E}_{p_{\mathcal{D}}(x)}
                 \mathbb{E}_{z_1, \ldots, z_k}
                \bigg[ \log \frac{1}{k}\sum^{k}_{j=1} p_\theta(x|z_i)
                \exp(-T_\psi(z_j))\bigg].
                \label{eqn:IW-AAE}
        \end{multline}}
    
    Bousquet {\em et al.} also show that minimizing $\mathcal{L}_{\text{AAE}}$ is a special case of minimizing a penalized optimal transport
    (POT) with $2$-Wasserstein distance.
        
    In the following sections, we will use the pseudo-divergences, Equantion~\ref{eqn:aiwb3} and Equation~\ref{eqn:IW-AAE}, to analyze IW-AVB and IW-AAE.




\subsection{Relationship between Wasserstein Autoencoders and log-likelihood} \label{sec:rwae}

We would like to understand the relationship between AAE (IW-AAE) and log-likelihood. 
Previously, it was shown that $\mathcal{L}_{\text{AAE}}$ converges to  Wasserstein autoencoder objective function $W_c^\dagger(p_\mathcal{D}, p_\theta)$ under certain circumstances \cite{Bousquet2017}.
We observe that $\mathcal{L}_{\text{IW-AAE}}$ converges to new Wasserstein autoencoder objective $W_c^\ddagger(p_\mathcal{D}, p_\theta)$ which gives a tighter bound on the autoencoder log-likelihood $\log p_{\phi,\theta}(x)=\int p_\theta(x|z)q_\phi(z|x)dz$. The quantity $\log p_{\phi,\theta}(x)$ can be understood as likelihood of reconstructed data from probabilistic encoding model. Further in Corollary in Appendix, in a special case, we were able to relate $\log p_{\phi,\theta}(x)$ and $\log p_{\theta}(x)$.

Wasserstein distance $W_c(p_\mathcal{D}, p_\theta)$ is a distance function defined between probability distribution on a a metric space. 
Bousquet {\em et al.} \cite{Bousquet2017} showed that the penalized optimal transportation objective $D_{\text{POT}}$ is relaxed version of Wasserstein autoencoder objective $W_c^\dagger(p_\mathcal{D}, p_\theta)$\footnote{Given that the generative network $p_\theta$ is probabilistic function, we have $W_c(p_\mathcal{D}, p_\theta) \leq W_c^\dagger(p_\mathcal{D}, p_\theta)$} where
\begin{align}
    D_{\text{POT}}(p_\mathcal{D}, p_\theta) := \inf_{q(z|x)\in\mathcal{Q}} \mathbb{E}_{p_\mathcal{D}(x)}\mathbb{E}_{q(z|x)} \left[ c(x,g(z))\right] + \lambda D_{\text{GAN}}(q(z),p(z)).
\end{align}
and $c$ is a distance function.
$D_{\text{GAN}}$ is used for the choice of convex divergence between the prior $p(z)$ and 
the aggregated posterior $q(z)$ \cite{Bousquet2017}.
As $\lambda \to \infty$, $D_{\text{POT}}(p_\mathcal{D}, p_\theta)$ converges to $W_c^\dagger(p_\mathcal{D}, p_\theta)$.  
It turns out that $D_{\text{AAE}}$ is a special case of $D_{\text{POT}}$.
This happens when the cost function $c$ is squared Euclidean distance and
$P_G(\tilde{x}|z)$ is Gaussian $\mathcal{N}(\tilde{x}; G_\theta(z),\sigma^2 I)$.

We can also observe that $D_{\text{IW-AAE}}$ converges to $W_c^\ddagger(p_\mathcal{D}, p_\theta)$:
        \begin{proposition}
            Assume that $c(x,\tilde{x}) = \|x-\tilde{x}\|^2$,
            $p_\theta(\tilde{x}|z)=\mathcal{N}(\tilde{x};G_\theta(z), \sigma^2I)$.
            Then, $- \log p_{\theta, \phi}(x) \leq W_c^\ddagger(p_\mathcal{D}, p_\theta) \leq W_c^\dagger(p_\mathcal{D}, p_\theta)$
            where $W_c^\ddagger(p_\mathcal{D}, p_\theta)$ and $W_c^\dagger(p_\mathcal{D}, p_\theta)$ are
            \begin{align*}
                 &\inf_{q:q_\phi(z)=p(z)} \mathbb{E}_{p_\mathcal{D}}
                  \mathbb{E}_{z_1\ldots z_k \sim q_\phi(z|x)} \left[ \log \frac{1}{K}\sum^{K}_{k=0}  p(x, G_\theta(z_k))\right]\\ 
                 & \qquad \qquad \qquad \qquad \qquad  \text{and }\\
                 &\inf_{q:q_\phi(z)=p(z)} \mathbb{E}_{p_\mathcal{D}}
                    \mathbb{E}_{z_1\ldots z_k \sim q_\phi(z|x)} \left[ \frac{1}{K}\sum^{K}_{k=0} \log p(x, G_\theta(z_k))\right],
            \end{align*}
            where $\log p_{\theta,\phi}$ is the log-likelihood of an autoencoder\footnote{We abuse the notation by writing $\log p_{\theta,\phi}(X=x|X^\prime=x)$ as a $\log p_{\theta,\phi}(x)$.}.
            Moreover, $D_{\text{AAE}}$ converges to $W_c^\dagger(p_\mathcal{D}, p_\theta)$
            and $D_{\text{IW-AAE}}$ converges to $W_c^\ddagger(p_\mathcal{D}, p_\theta)$
            as $\lambda=2\sigma^2\rightarrow \infty$.
            \label{prop:wasser}
        \end{proposition}

        The bound is derived by applying Jensen's inequality (see the proof in the Appendix).
        We observe that $W_c^\dagger(p_\mathcal{D}, p_\theta)$  is the lower bound of $\log p_{\theta,\phi}(x)$ under 
        the condition that $p(z)=q_\phi(z)$.
        The tighter bound is achieve using $W_c^\ddagger(p_\mathcal{D}, p_\theta)$ compare to $W_c^\dagger(p_\mathcal{D}, p_\theta)$.
        Lastly, we observe that $D_{\text{AAE}}$ approximates 
        $W_c^\dagger(p_\mathcal{D}, p_\theta)$ and
        $D_{\text{IW-AAE}}$ approximates $W_c^\ddagger(p_\mathcal{D}, p_\theta)$.

        The following theorem shows the relationship between AAE objective and $\log p_\theta(x)$.
       
        \begin{theorem}
        \label{aae:theorem1}
        Maximizing AAE objective is equivalent to jointly maximizing 
        $\log p_\theta(x)$, mutual information with respect to 
        $q_\phi(x,z)$, and the negative of KL divergence between 
        joint distribution $p_\theta(x,z)$ and $q_\phi(x,z)$,
        \begin{align} 
            \mathcal{L}_{\text{AAE}} = \log p_\theta(x) + \mathbb{I}_{q_\phi(x,z)}\left[x,z\right]-\mathbb{KL}(q_\phi(x,z)\|p(x,z)).
            \label{eqn:theorem1}
        \end{align} 
        \end{theorem}
        The proof is in the Appendix. 
        This illustrate the trade of between the mutual information and the relative information between $q_\phi(x,z)$ and $p(x,z)$.
        In order for the gap between $\log p_\theta(x)$ and $\mathcal{L}_{\text{AAE}}$ to be small, $q_\phi(z)p(x)$ need to become close to $p(x,z)$.

        \begin{theorem}
        \label{aae:theorem2}
        The difference between maximizing AAE versus VAE corresponds to
        \begin{align} 
            \mathcal{L}_{\text{AAE}} - \mathcal{L}_{\text{VAE}} = 
                \mathbb{E}_{p_\mathcal{D}(x)}\left[\mathbb{KL}(q_\phi(z|x)\|q_\phi(z))\right].
        \end{align} 
        \end{theorem}
        The proof is in the Appendix. 
        It is interesting to observe that the gap between two objectives is the KL 
        divergence between approximate posterior distribution 
        $q_\phi(z|x)$ and the marginal distribution $q_\phi(z)$.

\subsection{Relationship of IW-AVB and IW-AAE to other objectives} 


The two adversarial objectives bound becomes a tighter upper-bound as the number of samples increases:
\begin{proposition}
    For any distribution $p_{\mathcal{D}}(x)$ and $p_\theta(x)$, and for $k > m$ samples:
    \begin{align}
        -\log p_\theta(x) &\leq D_{\text{IW-AVB}}^k(p_{D}, p_\theta)  \leq D_{\text{IW-AVB}}^m(p_{D}, p_\theta) \text{ and } \label{eqn:iwavb_ine}\\
        -\log p_\theta(x) +  \mathbb{I}_{q_\phi(x,z)}\left[x,z\right] &\leq D_{\text{IW-AAE}}^k(p_{D}, p_\theta) \leq D_{\text{IW-AAE}}^m(p_{D}, p_\theta) \label{eqn:iwaae_ine}
    \end{align}
    \label{prop:iwaae}
    \vspace{-0.5cm}
\end{proposition}

Following the steps of the proof from Theorem~1 in \cite{Burda2015},
we can show that the bound becomes tigher for both 
$D_{\text{IW-AVB}}^k(p_{D}, p_\theta)$ and $D_{\text{IW-AAE}}^k(p_{D}, p_\theta)$
as you get more samples.
For left inequality of Equation~\ref{eqn:iwavb_ine}, IW-AVB appraoches to the 
log-likelihood since IWAE approaches $\log p(x)$.
For left inequality of Equation~\ref{eqn:iwaae_ine}, we know that 
$-D_{\text{IW-AAE}}^k(p_{D}, p_\theta)$ gets tigher as we increase the number of
 samples, and yet it's upper bounded by 
 $\log p_\theta(x) + \mathbb{I}_{q_\phi(x,z)}\left[x,z\right]$ from Theorem~\ref{aae:theorem1}
 (we omit the term $-\mathbb{KL}(q_\phi(x,z)\|p(x,z))$ for the upper bound).


The relationships between 
$D_{\text{IW-AVB}}$ and $D_{\text{IW-AAE}}$, $D_{\text{AAE}}$, and $D_{\text{AVB}}$
are
\begin{proposition}
    For any distribution $p_{\mathcal{D}}(x)$ and $p_\theta(x)$:
    \begin{align*}
        D_{\text{IW-AAE}}(p_{D}, p_\theta) \leq D_{\text{AAE}}(p_{D}, p_\theta) \leq D_{\text{AVB}}(p_{D}, p_\theta)\\
        D_{\text{IW-AAE}}(p_{D}, p_\theta) \leq D_{\text{IW-AVB}}(p_{D}, p_\theta) \leq D_{\text{AVB}}(p_{D}, p_\theta)
        \label{prop3}
    \end{align*}
\end{proposition}
The proof is shown in the Appendix. 
The $D_{\text{IW-AAE}}$ is tighter than $D_{\text{AAE}}$ (Proposition~\ref{prop:iwaae}), and the $D_{\text{IW-AAE}}$ is tighter than $D_{\text{IW-AVB}}$ due to tighter adversarial approximation
(i.e., $D_{\text{GAN}}(\int q_\phi(z|x)p_{\mathcal{D}}(x)dx, p(z)) \leq 
\mathbb{E}_{p_{\mathcal{D}}(x)}\left[D_{\text{GAN}}(q_\phi(z|x)p_\mathcal{D}(x)dx, p(z))\right]$
since $D_{\text{GAN}}$ is convex).
However, the relationship between $D_{\text{IW-AVB}}$ and $D_{\text{AAE}}$ 
is unknown, because the trade-off between importance 
weighting bound versus the more flexible adversarial objective is unclear.

Remark that IW-AVB guarantees to be the lower bound of the log-likelihood.
However, both AAE and IW-AAE are not guarantee to be the lower bound of the 
log-likelihood due to the extra mutual information term $\mathbb{I}_{q_\phi(x,z)}\left[x,z\right]$ in Equation~\ref{eqn:theorem1}.
This illustrates that blindly increasing the number of samples for IW-AAE is not
 a good idea, rather we have to choose the number of samples such that it balances 
 between $\mathbb{I}_{q_\phi(x,z)}\left[x,z\right]$ and $\mathbb{KL}(q_\phi(x,z)\|p(x,z))$.
Nevertheless depending on the metric, you will see that IW-AAE can perform better
for some tasks in our experiment section.

\begin{figure*}[t]
    \centering
    \begin{minipage}{0.245\textwidth}
        \includegraphics[width=0.99\linewidth]{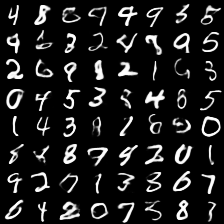}
        \subcaption{AVB}
    \end{minipage}
    \begin{minipage}{0.245\textwidth}
        \includegraphics[width=0.99\linewidth]{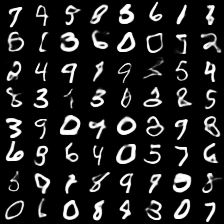}
        \subcaption{IW-AVB}
    \end{minipage}
    \begin{minipage}{0.245\textwidth}
        \includegraphics[width=0.99\linewidth]{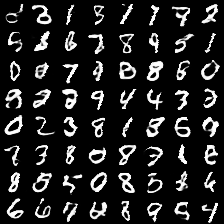}
        \subcaption{AAE}
    \end{minipage}
    \begin{minipage}{0.245\textwidth}
        \includegraphics[width=0.99\linewidth]{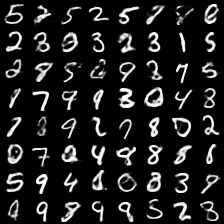}
        \subcaption{IW-AAE}
    \end{minipage}\\
    \caption{Samples of generative models from training MNIST dataset. }
    \label{fig:mnist_samples}
\end{figure*}

\begin{figure*}[t]
    \centering
    \hspace{-0.5cm}
    \begin{minipage}{0.37\textwidth}
        \includegraphics[width=0.99\linewidth]{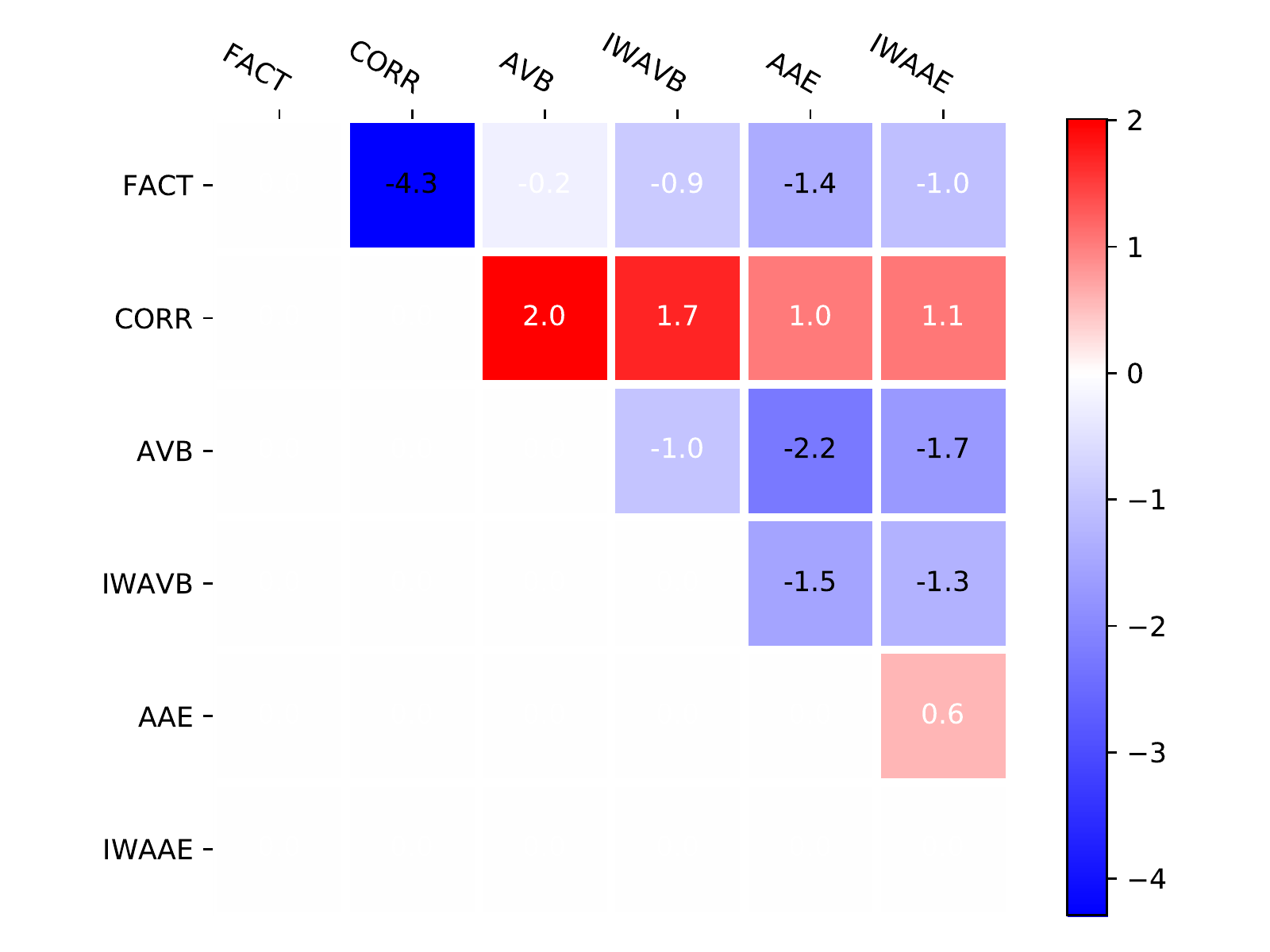}
        \subcaption{Non-amortized spike inference}
        \label{fig:ttest_a}
    \end{minipage}
    \begin{minipage}{0.37\textwidth}
        \includegraphics[width=0.99\linewidth]{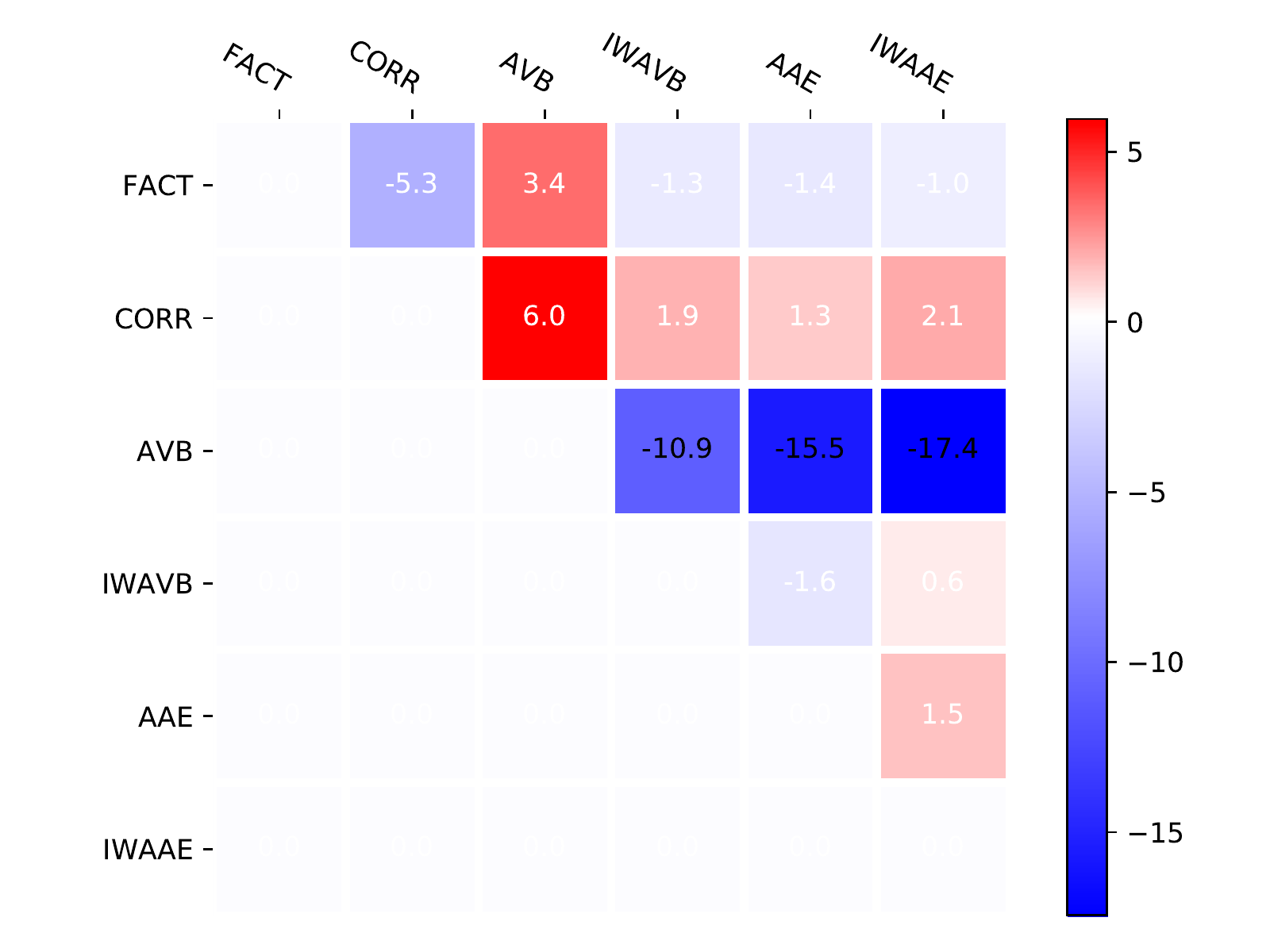}
        \subcaption{Amortized spike inference}
        \label{fig:ttest_b}
    \end{minipage}
    \begin{minipage}{0.275\textwidth}
        \includegraphics[width=\linewidth]{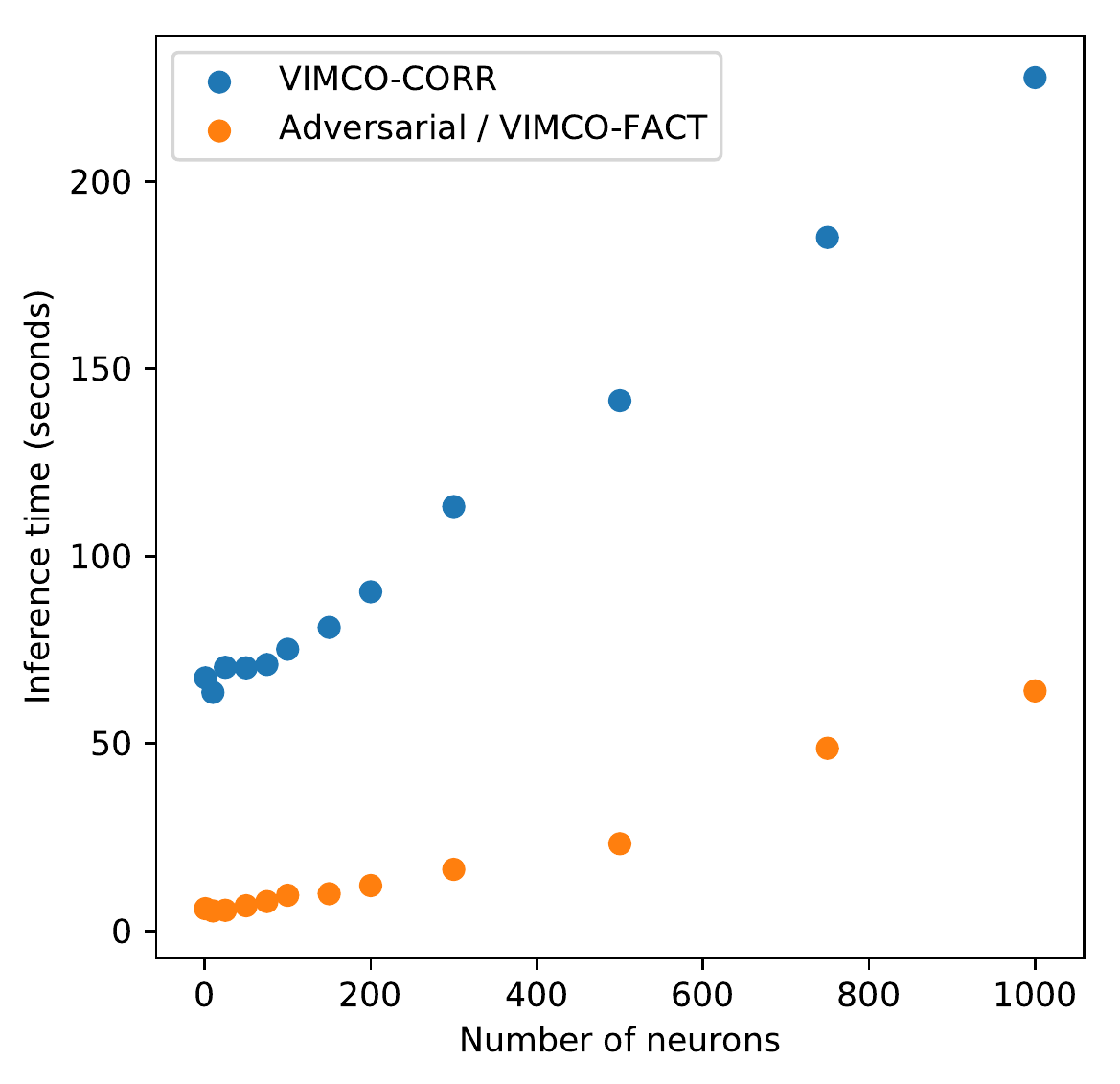}
        \subcaption{Inference time.}
        \label{fig:time_analysis}
    \end{minipage}
    \caption{Pairwise performance comparison using t-test}
    \label{fig:ttest}
\end{figure*}

\section{Experiments}
\label{sec:experiments}

We conducted our experiments with two main objectives where we want to
i. compare the performance between AVB, IW-AVB, AAE, and IW-AAE;
ii. check whether the adversarial training objectives can benefit neural spike inference in general.
For such reasons, we measure their performance in two experimental setups.
First, we experiment on generative modeling task on MNIST dataset. 
Second, we apply adversarial training on neuron spike activity inference dataset with both amortized and non-amortized inference settings.

\subsection{Generative modeling}
We follow the same experimental procedure as 
    \cite{Mescheder2017} for learning generative models on binarized MNIST dataset.
    We trained AVB, IW-AVB, AAE, and IW-AAE on 50,000 train examples with 10,000 validation examples,
    and measured log-likelihood on 10,000 test examples.
    We applied the same architecture from \cite{Mescheder2017}\footnote{We followed the experiment and the code from  \texttt{https://github.com/LMescheder/AdversarialVariationalBayes}.}. 
    See the details of \cite{Mescheder2017} in Supplementary Materials.


        We considered three following metrics. 
        The log-likelihood was computed using Annealed Importance Sampling 
        (AIS) \cite{Neal2001, Wu2016} with 1000 intermediate distribution 
        and 5 parallel chains.
        We also applied the Frechet Inception Distance (FID)
        \cite{Heusel2017}. 
        It compares the mean $m_Q$ and covariance $C_Q$ of the Inception-based 
        representation of samples generated by the GAN to the mean $m_P$ and 
        covariance $C_P$ of the same representation for training samples:
        \begin{align}
            D^2\left((m_P,C_P),(m_Q,C_Q)\right) = \| m_P - m_Q\|^2_2 
            + \text{Tr}\left(C_P+ C_Q - 2(C_PC_Q)^{\frac{1}{2}}\right),
        \end{align}
        Lastly, we also considered GAN metric proposed by \cite{Im2018}
        that measure the quality of generator by estimating divergence 
        between the true data distribution $P$ and $Q_\theta$ 
        for different choices of divergence measure. 
        In our setting we considered least-square measure (LS). 

    \begin{wraptable}{r}{0.5\textwidth}
        \vspace{-0.5cm}
  \centering
  \caption{Log-likelihood, FID, and LS metric on binarized MNIST.
    IW-AVB performs best for log-likelihood and IW-AAE performs best under FID and LS metric.}
  \label{tab:results}
  \begin{tabular}{|l|ccc|}\hline
                            & Log-Likelihood  & FID & LS \\\hline\hline 
      VAE                   & -90.69 $\pm$ 0.88 & 259.87       & 3.8e-5 \\
      IWAE                  & -91.64 $\pm$ 0.71 & 255.513      & 3.6e-5 \\
      AVB                   & -90.42 $\pm$ 0.78 & 256.13       & 4.1e-5 \\
      IW-AVB                & {\bf -85.12 $\pm$  0.20} & 251.20       & 3.3e-5 \\
      AAE                   & -101.78 $\pm$ 0.62       & 266.76       & 3.8e-5 \\ 
      IW-AAE                & -101.38 $\pm$ 0.19       & {\bf 249.12} & {\bf 3.2e-5}\\\hline
  \end{tabular}
\end{wraptable}

        Table~\ref{tab:results} presents the results. 
        We observe that IW-AVB gets the best test log-likelihood for
        both MNIST and FashionMNIST dataset\footnote{Note that our results are slightly lower than the reported results in \cite{Mescheder2017}. However, we used same codebase for all models}  (the results for FashionMNIST is shown in Appendix). 
        On the other hand, IW-AAE gets the best FID and LS metric. 
        We speculate that the reason is because AVB and IW-AVB directly 
        maximizes the lower bound of the log-likelihood $\log p_\theta(x)$, whereas
        AAE and IW-AAE does not. AAE and IW-AAE maximizes the 
        distance between data and model distribution directly. 
        The MNIST and FashionMNIST samples are shown in Figure~\ref{fig:mnist_samples} and ~\ref{fig:fmnist_samples} in Appendix.

\begin{table*}[t]
    \centering
    \caption{Non-amortized Spike Inference - The performance comparison. Values are correlation between predicted marginal probabilities and ground truth spikes at 25Hz}
    \label{tab:neuron_results_corr25}
     {\footnotesize
    \begin{tabular}{|l|ccccc|c|}\hline
                             & Neuron1 & Neuron2 & Neuron3 & Neuron4 & Neuron5 & Avg.  \\\hline\hline
        VIMCO-FACT                  & 0.653 $\pm$ 0.02 & 0.631 $\pm$ 0.017 & 0.613 $\pm$ 0.026 & 0.473 $\pm$ 0.03 &	0.585 $\pm$ 0.03 &	0.590 $\pm$ 0.063 \\
        VIMCO-CORR                 & 0.711 $\pm$ 0.008  &	0.665 $\pm$ 0.007 &	0.704 $\pm$ 0.01 &	0.50 $\pm$ 0.017 &	0.623 $\pm$ 0.01 &	{\bf 0.64 $\pm$ 0.077} \\\hline
        AVB                   & 0.631 $\pm$ 0.022 & 0.617 $\pm$ 0.003 &	0.617 $\pm$ 0.010 &	0.540 $\pm$ 0.012 &	0.570 $\pm$ 0.003 &	0.594 $\pm$ 0.034 \\
        IW-AVB                & 0.681 $\pm$ 0.005 &	0.616 $\pm$ 0.005 & 0.613 $\pm$ 0.006  &  0.577 $\pm$ 0.005 & 0.567 $\pm$ 0.002 & 0.611 $\pm$ 0.040	 \\\hline 
        AAE                   & 0.680 $\pm$ 0.002 &	0.617 $\pm$ 0.002 & 0.66 $\pm$ 0.01 & 0.563 $\pm$ 0.001  &	0.570 $\pm$ 0.005 &	0.618 $\pm$ 0.047 \\        
        IW-AAE                &  0.682 $\pm$ 0.002   &  0.622 $\pm$  0.002 & 0.614 $\pm$ 0.006 & 0.570 $\pm$ 0.002	 &	0.573 $\pm$ 0.003 &	0.612 $\pm$ 0.041 \\\hline
    \end{tabular}}

    \vspace{0.5cm}
    \centering
    \caption{Amortized Spike Inference - The performance comparison. Values are correlation between predicted marginal probabilities and ground truth spikes at 25Hz}
    \label{tab:neuron_results_corr60}
     {\footnotesize
    \begin{tabular}{|l|ccccc|c|}\hline
                             & Neuron1 & Neuron2 & Neuron3 & Neuron4 & Neuron5 & Avg. \\\hline\hline
        VIMCO-FACT            & 0.583 $\pm$ 0.02 & 0.594 $\pm$ 0.014 & 0.676 $\pm$ 0.026 & 0.613 $\pm$ 0.03 & 0.560 $\pm$ 0.014 & 0.606 $\pm$ 0.044\\
        VIMCO-CORR            & 0.620 $\pm$ 0.023  & 0.664 $\pm$ 0.014 & 0.708 $\pm$ 0.01	 & 0.642 $\pm$ 0.014   &  0.596 $\pm$ 0.01 &	0.646 $\pm$ 0.042 \\\hline
        AVB                   & 0.611 $\pm$ 0.016 &  0.554 $\pm$ 0.02  & 0.558 $\pm$ 0.002	 & 0.517 $\pm$ 0.017	 &	0.508 $\pm$ 0.024 & 0.550 $\pm$ 0.036	 \\
        IW-AVB                & 0.689 $\pm$ 0.002 &	0.621 $\pm$ 0.004 & 0.610 $\pm$ 0.004  & 0.587 $\pm$ 0.002 & 0.60 $\pm$ 0.008 & 0.621 $\pm$ 0.035 \\\hline 
        AAE                   & 0.691 $\pm$ 0.003 & 0.624 $\pm$ 0.002 & 0.624 $\pm$ 0.003 & 0.587 $\pm$ 0.004 & 0.601 $\pm$ 0.007 & {\bf 0.625 $\pm$ 0.036} \\        
        IW-AAE                & 0.688 $\pm$ 0.001 & 0.624 $\pm$ 0.003 & 0.612 $\pm$ 0.009 & 0.590 $\pm$ 0.004	 &	0.580 $\pm$ 0.004 & 0.619 $\pm$ 0.038	 \\\hline
    \end{tabular}}
\end{table*}

\subsection{Neural activity inference from calcium imaging data}
We consider a challenging and important problem in neuroscience -- spike inference from calcium imaging. Here, the unobserved binary spike train is the latent variable which is transformed by a generative model whose functional form is derived from biophysics into fluorescence measurements of the intracellular calcium concentration. 



We use a publicly available spike inference dataset, cai-1\footnote{The dataset is available at {\em https://crcns.org/data-sets/methods/cai-1/}}. We use the data from five layer 2/3 pyramidal neurons in mouse visual cortex \footnote{We excluded neurons that has clear artifacts and mislabels in the dataset.}. The neurons are imaged at 60 Hz using  GCaMP6f -- a genetically encoded calcium indicator \cite{chen2013ultrasensitive}. The ground truth spikes were measured electrophysiologically using cell-attached recordings.  


When we train AVB, AAE, IW-AVB, IW-AAE to model fluorescence data, we use a biophysical generative model and a convolutional neural network as our inference network.  Thus, the process is to generate (reconstruct) the fluorescence traces with inferred spikes using encoders. We ran five folds on every experiments in neural spike inference dataset. The details of architectures, biophysical model, and datasets can be found in the Appendix~\ref{spike_inf}

We experimented under two settings: Non-amortized spike inference, and amortized spike inference settings. 
Non-amortized spike inference corresponds to training a new inference network for each neuron. This is expensive but it provides an estimate of the best possible performance achievable.
Amortized spike inference setup corresponds to the more useful setting where a ``training'' dataset of neurons is used to train an inference network (without ground truth), and the trained inference network is tested on a new ``test'' neuron. This is the more practically useful setting for spike inference -- once the inference network is trained, spike inference is extremely fast and only requires prediction by the inference network. 

We use two variants of VIMCO \cite{Mnih2016} as a baseline, VIMCO-FACT and VIMCO-CORR\cite{speiser2017fast}. VIMCO-FACT uses a fast factorized posterior distribution which can be sampled in parallel over time, same as the adversarially trained networks. VIMCO-CORR uses an autoregressive posterior that produces correlated samples which must be sampled sequentially in time (see the details in the Appendix~\ref{spike_inf}).


Following the neuroscience community, we evaluated the quality of our posterior inference networks by computing the correlation between predicted spikes and labels as the performance metric. 
We used a paired t-test \cite{goulden1956} compare the improvement of all pairs of inference networks across five neurons (see Figure~\ref{fig:ttest}).
The full table of correlations scores for all neurons and methods in both amortized and non-amortized settings are shown in Appendix Table~\ref{tab:neuron_results_corr25}. 
We observe that AVB, AAE, IW-AAE, and IW-AVB performances lie in between VIMCO-FACT and VIMCO-CORR.
Overall, we observe that IW-AVB, AAE, and IW-AAE performs similarly across given neuron datasets.
Figure~\ref{fig:reconstruction} illustrates the VIMCO-FACT and IW-AVB posterior approximation on {\em neuron 1 dataset}.
From the figure, we observe that VIMCO-FACT tend to have high false negatives while IW-AVB tend to have high false positives.
The results are similar for amortized experiments as shown in Table~\ref{tab:neuron_results_corr60}. 
Interestingly, the performance of IW-AVB, AAE, and IW-AAE were better than non-amortized experiments. 
This suggests that neural spike influencing can be generalized over multiple neurons. 
Note that this is the first time that adversarial training has been applied to neural spike inference.

\begin{figure*}[t]
    \centering
    \includegraphics[width=\linewidth]{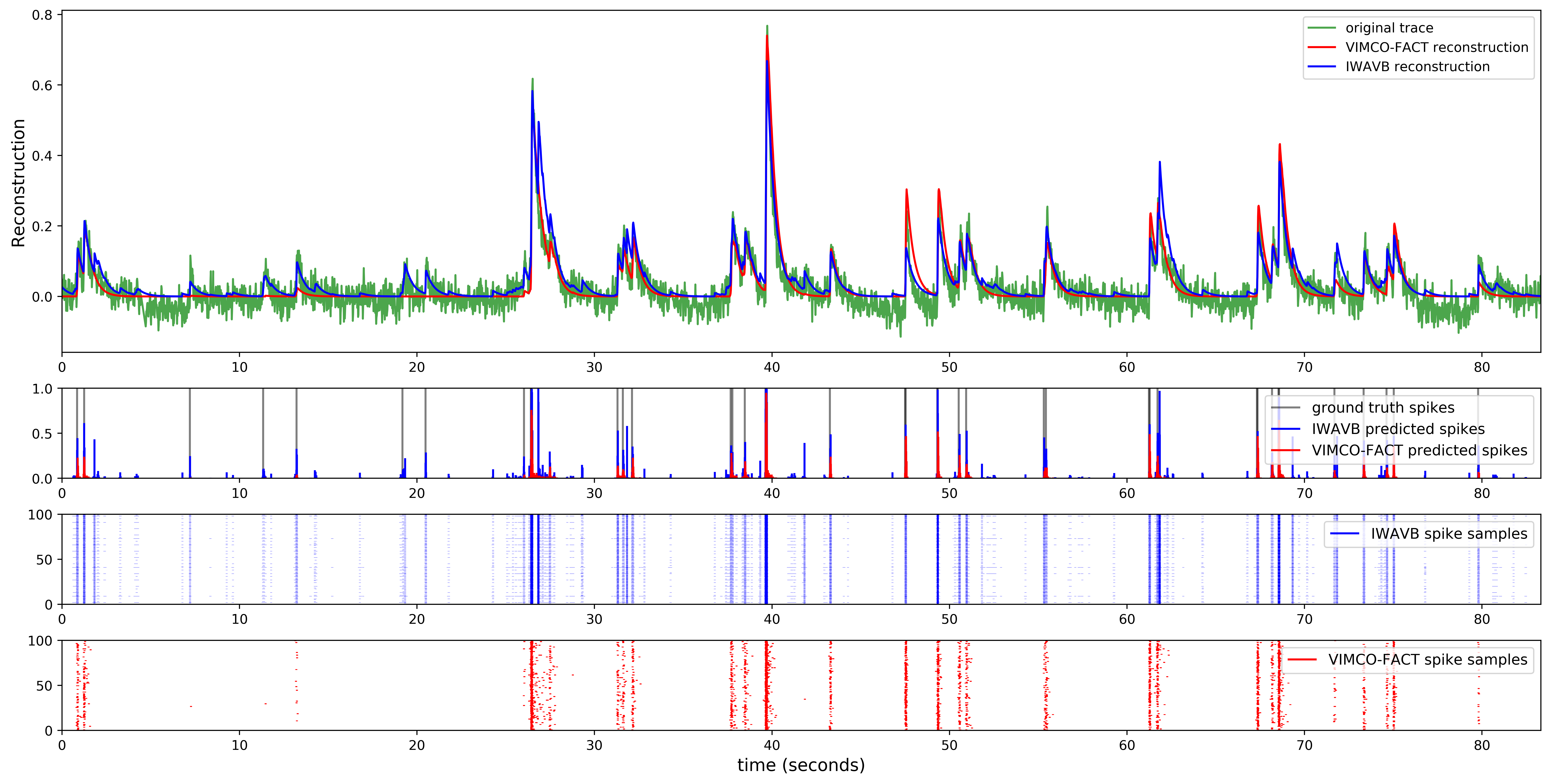}
    \caption{Trace reconstructions along with spike samples}
    \label{fig:reconstruction}
\end{figure*}

Moreover, VIMCO-CORR generates correlated posterior samples, whereas the samples from VIMCO-FACT are independent.
Nevertheless, the inference is slower at test time compared to VIMCO-FACT since the spike inferences 
are done sequentially rather than in parallel.  This is huge disadvantage to VIMCO-CORR, because spike records can be hourly long.
We emphasize the adversarial training, such as IW-AVB  and IW-AAE, because they generate correlated posterior samples in parallel. 
Figure~\ref{fig:time_analysis} demonstrates the time advantage of adversarial training over VIMCO-CORR.
The total florescence data duration was 1 hour at a 60 Hz sampling rate and ran on NVIDIA GeForce RTX 2080 Ti.

\section{Conclusions}
Motivated by two ways of improving the variational bound: importance weighting \cite{Burda2015} and better posterior approximation \cite{Rezende:2015vu,Mescheder2017,Makhzani2016}, we propose importance weighted adversarial variational Bayes (IW-AVB) and importance weighted adversarial autoencoder (IW-AAE). Our theoretical analysis provides better understanding of adversarial autoencoder objectives,
and bridges the gap between log-likelihood of an autoencoder and generator.  


Adversarially trained inference networks are particularly effective at learning correlated posterior distributions over discrete latent variables which can be sampled efficiently in parallel. We exploit this finding to apply both standard and importance weighted variants of AVB and AAE to the important yet challenging problem of inferring spiking neural activity from calcium imaging data.
We have empirically shown that the correlated posteriors trained adversarially in general outperform existing VAEs with factorized posteriors.
Moreover, we get tremendous speed gain during the spike inference compare to existing VAEs work with autoregressive correlated posteriors \cite{speiser2017fast}.
\vfill

\pagebreak

\bibliography{main}
\bibliographystyle{plain}

\setcounter{proposition}{0} 
\setcounter{theorem}{0} 
\onecolumn
\section*{Appendix}

\subsection*{Proof of Things}

\begin{proposition}
    Assume that $T$ can represent any function of two variables.
    If $(\theta^*, \phi^*, T^*)$ is a Nash Equilibrium of two-player game,
    then $T^*(x,z) = \log \frac{p(z)}{q(z|x)}$
    and $(\theta^*, \phi^*)$ is a global optimum of the importance weighted lower bound.
\end{proposition}

\begin{proof}
    Suppose that $(\theta^*, \phi^*, T^*)$ is a Nash Equilibrium.
    It was previously shown by \cite{Goodfellow2014} that 
    \begin{align*}
    T^*(x,z) 
    = \max_{T}
        \mathbb{E}_{p_{\mathcal{D}}(x)}\left[
            \mathbb{E}_{q_\phi(z|x)} \bigg[ \log \sigma (T_\psi(x,z)) \right]
        + \mathbb{E}_{p(z)} \log (1- \sigma(T_\psi(x,z)))\bigg]
    = \log p(z) - \log q_\phi(z|x).
    \end{align*}
    
    Now, we substitute
    $T^*(x,z)=\log q^*(z|x) - \log p(z)$ 
    into Equation~\ref{eqn:aiwbo} and show that $(\theta^*, \phi^*)$  
    maximizes the following formula as a function of $\phi$ and $\theta$:
    \begin{align*}
        & \mathbb{E}_{p_D(x)} \mathbb{E}_{q_\phi(z|x)} 
        \left[ \log \frac{1}{K}\sum^K_{i=1} p_\theta(x|z_i)\exp(-T^*(x,z_i)) \right]
        = 
        \mathbb{E}_{p_D(x)} \mathbb{E}_{z\sim q_\phi(z|x)} 
        \left[ \log \frac{1}{K}\sum^K_{i=1} p_\theta(x|z_i)\frac{p(z_i)}{q_{\phi^*}(z_i|x)} \right].
    \end{align*}

    Define the implicit distribution $q_{\text{IW}}$:
        \begin{align*}
            q_{IW}(z_i|x,z_{/\ i}) := K\tilde{w}_iq(z_i|x) = \frac{p(x,z_i)}{\frac{1}{k}\sum^K_{j=1} \frac{p(x,z_j)}{q(z_j|x)}},
        \end{align*}
        where 
        \begin{align*}
            \tilde{w}_i = \frac{\frac{p(x,z_i)}{q(z_i|x)}}{\sum^K_{j=1} \frac{p(x,z_j)}{q(z_j|x)}}
        \end{align*}
        is an importance weight.

    Now, following the steps of turning $\mathcal{L}_{\text{IWAE}}$ in terms of $\mathcal{L}_{\text{VAE}}$ with implicit $q_{\phi^*}^{IW}$   \cite{Cremer2017}, we have
    \begin{align*}
        \mathbb{E}_{p_D(x)} \mathbb{E}_{z\sim q_\phi(z|x)} \left[ \log \frac{1}{K}\sum^K_{i=1} \frac{p_\theta(x,z_i)}{q_{\phi^*}(z_i|x)} \right] 
        &= \mathbb{E}_{p_D(x)} \mathbb{E}_{z\sim q_\phi(z|x)} \left[ \sum^{K}_{l=1} \tilde{w_l} \left[ \log \frac{1}{K}\sum^K_{i=1}           \frac{p_\theta(x,z_i)}{q_{\phi^*}(z_i|x)} \right]\right] \\
        &= \mathbb{E}_{p_D(x)} \mathbb{E}_{z\sim q_\phi^{IW}(z|x)} \left[ \log \frac{1}{K}\sum^K_{i=1}           \frac{p_\theta(x,z_i)}{q_{\phi^*}(z_i|x)} \right] \\
        &= \mathbb{E}_{p_D(x)} \mathbb{E}_{z\sim q_\phi^{IW}(z|x)} \left[ \frac{1}{K}\sum^{K}_{j=1} \log \frac{p(x,z_j)}{\frac{p(x,z_j)}{\frac{1}{K}\sum^K_{i=1}\frac{p_\theta(x,z_i)}{q_{\phi^*}(z_i|x)}}} \right] \\
        &= \mathbb{E}_{p_D(x)}\mathbb{E}_{q_\phi^{IW}}(z|x) 
        \left[ \frac{1}{K}\sum^{K}_{i=1}\log \left(\frac{p_\theta(x,z_i)}{q_{\phi^*}^{IW}(z_i|x)}\right)\right]
    \end{align*} 
    
    Thus, $(\theta^*, \phi^*)$ maximizes
    \begin{align*}
        J(\theta, \phi) = \mathcal{L}_{VAE}[q_{\text{IW}}](\theta, \phi) + \mathbb{E}_{p_D(x)}\left[\mathbb{KL}(q_\phi^{IW}(z|x), q_\phi^{IW*}(z|x))\right].
    \end{align*}

    Proof by contradiction, suppose that $(\theta^*, \phi^*)$ does not
    maximize the variational lower bound in $\mathcal{L}_{VAE}[q_{\text{IW}}](\theta, \phi)$.
    So there exist $(\theta^\prime, \phi^\prime)$ such that 
    \begin{align*}
        \mathcal{L}_{VAE}[q_{\text{IW}}](\theta^\prime, \phi^\prime)
        > \mathcal{L}_{VAE}[q_{\text{IW}}](\theta^*, \phi^*).
    \end{align*}
    However, substituting $\mathcal{L}_{VAE}[q_{\text{IW}}](\theta^\prime, \phi^\prime)$in 
    $J(\theta, \phi)$ is greater than $J(\theta^*, \phi^*)$, which contradicts the assumption.
    Hence, $(\theta^*, \phi^*)$ is a global optimum of 
    $\mathcal{L}_{VAE}[q_{\text{IW}}](\theta, \phi)$.

    Since we can express $\mathcal{L}_{\text{IWAE}}$ in terms of 
        $\mathcal{L}_{\text{VAE}}$ with implicit distribution $Q_{\text{IW}}$ \cite{Cremer2017},
        \begin{align*}
            \mathcal{L}_{\text{VAE}}[\text{Q}_{\text{IW}}] \nonumber
            &= \mathbb{E}_{p_{\mathcal{D}(x)}}\mathbb{E}_{z_1\cdots z_k\sim q_{IW}(z|x)}\left[ \frac{1}{k}\sum^K_{i=1} \log\left(\frac{p(x,z_i)}{q_{IW}(z_i|x)}\right)\right]\nonumber\\
            &= \mathbb{E}_{p_{\mathcal{D}(x)}}\mathbb{E}_{z_1\cdots z_k\sim q(z|x)}\left[ \log\left(\frac{1}{k}\sum^K_{i=1} \frac{p(x,z_i)}{q(z_i|x)}\right)\right]
            = \mathcal{L}_{\text{IWAE}}
        \end{align*}

    ($\theta^*, \phi^*$) is also a global optimum of the importance weighted lower bound. 
    
\end{proof}

\setcounter{proposition}{2} 
\begin{proposition}
    Assume that $c(x,\tilde{x}) = \|x-\tilde{x}\|^2$,
    $p_\theta(\tilde{x}|z)=\mathcal{N}(\tilde{x};G_\theta(z), \sigma^2I)$.
    Then, 
    \begin{align} 
        \log p_{\theta,\phi}(x) \geq -W_c^\ddagger(p_\mathcal{D}, p_\theta) \geq -W_c^\dagger(p_\mathcal{D}, p_\theta)
    \end{align}
    where
    \begin{align*}
        W_c^\ddagger(p_\mathcal{D}, p_\theta) &= \inf_{q:q_\phi(z)=p(z)} \mathbb{E}_{p_\mathcal{D}}
            \mathbb{E}_{z_1\ldots z_k \sim q_\phi(z|x)} \left[ \log \frac{1}{K}\sum^{K}_{k=0}  p(x, G_\theta(z_k))\right] \\
        W_c^\dagger(p_\mathcal{D}, p_\theta) &= \inf_{q:q_\phi(z)=p(z)} \mathbb{E}_{p_\mathcal{D}}
            \mathbb{E}_{z_1\ldots z_k \sim q_\phi(z|x)} \left[ \frac{1}{K}\sum^{K}_{k=0} \log p(x, G_\theta(z_k))\right].
    \end{align*} 
    $D_{\text{AAE}}$ converges to $W_c^\dagger(p_\mathcal{D}, p_\theta)$
    and $D_{\text{IW-AAE}}$ converges to $W_c^\ddagger(p_\mathcal{D}, p_\theta)$
    as $\lambda=2\sigma^2\rightarrow \infty$.
\end{proposition}

\begin{proof}
    \begin{align*}
        -W^{\dagger} &= \inf_{q:q_\phi(z)=p(z)} \mathbb{E}_{p_\mathcal{D}}
            \mathbb{E}_{z_1\ldots z_k \sim q_\phi(z|X^\prime=x)} \left[ \frac{1}{K}\sum^{K}_{k=0} \log p_\theta(X=x|z_k)\right]\\
        &\leq \inf_{q:q_\phi(z)=p(z)} \mathbb{E}_{p_\mathcal{D}}
            \mathbb{E}_{z_1\ldots z_k \sim q_\phi(z|X^\prime=x)} \left[ \log \frac{1}{K}\sum^{K}_{k=0} p_\theta(X=x|z_k)\right]\\
        &\leq \inf_{q:q_\phi(z)=p(z)} \mathbb{E}_{p_\mathcal{D}}
            \left[ \log \mathbb{E}_{z_1\ldots z_k \sim q_\phi(z|X^\prime=x)} \frac{1}{K}\sum^{K}_{k=0} p_\theta(X=x|z_k)\right]\\
        &= \inf_{q:q_\phi(z)=p(z)} \mathbb{E}_{p_\mathcal{D}} \left[ \log \int q_\phi (z|X^\prime=x) p_\theta(X=x|z) dz \right]\\
        &= \inf_{q:q_\phi(z)=p(z)} \mathbb{E}_{p_\mathcal{D}} \left[ \log p_{\theta,\phi} (X=x|X^\prime=x) \right]
    \end{align*}
    
    As the $\lambda$ goes to $\infty$, the constraint $q_\phi(z)=p(z)$ is satisfied. This means that $\log \frac{q(z)}{p(z)}=1$. 
    Then, AAE objectives becomes $W_c^\dagger(p_\mathcal{D}, p_\theta)$ and IW-AAE becomes $W_c^\ddagger(p_\mathcal{D}, p_\theta$).
    \end{proof}

\begin{theorem}
Maximizing AAE objective is equivalent to jointly maximizing 
$\log p_\theta(x)$, mutual information with respect to 
$q_\phi(x,z)$, and the negative of KL divergence between 
joint distribution $p_\theta(x,z)$ and $q_\phi(x,z)$,
\begin{align} 
    \mathcal{L}_{\text{AAE}} = \log p_\theta(x) + \mathbb{I}_{q_\phi(x,z)}\left[x,z\right]-\mathbb{KL}(q_\phi(x,z)\|p(x,z)).
\end{align}          
\label{theorem1}
\end{theorem}
\begin{proof}
    \begin{align}
        \mathcal{L}_{\text{AAE}} = \mathbb{E}_{p_\mathcal{D}(x)}\mathbb{E}_{q_\phi(z|x)} \left[ \log \frac{p_\theta(x|z)p(z)}{q_\phi(z)}\right] 
        \label{eqn:aae_org2}
    \end{align}
    is the AAE objective with optimal $T^*(z) = \frac{p(z)}{q(z)}$
    (In practice, we do adversarial training to approximate 
    Equation~\ref{eqn:aae_org2}).

    It is straight forward to see that 
    $\mathbb{E}_{p_\mathcal{D}(x)}\log p_\theta(x) = \mathbb{E}_{p_\mathcal{D}(x)}\left[ \mathbb{E}_{q_\phi(z|x)} \left[ \log \frac{p_\theta(x|z)p(z)}{q_\phi(z)}\right]
                + \mathbb{E}_{q_\phi(z|x)} \left[ \log \frac{q_\phi(z)}{p(z|x)}\right]  \right]$:
    \begin{align*}
        \mathbb{E}_{q_\phi(z|x)} \left[ \log \frac{p_\theta(x|z)p(z)}{q_\phi(z)}\right] 
                + \mathbb{E}_{q_\phi(z|x)} \left[ \log \frac{q_\phi(z)}{p(z|x)}\right]
            &=\mathbb{E}_{q_\phi(z|x)} \left[ \log \frac{p_\theta(x|z)p(z)}{q_\phi(z)}\frac{q_\phi(z)}{p(z|x)}\right]  \\
            &=\mathbb{E}_{q_\phi(z|x)} \left[ \log \frac{p_\theta(x|z)p(z)}{p(z|x)}\right]  \\
            &=\mathbb{E}_{q_\phi(z|x)} \left[ \log p_\theta(x) \right]  \\
            &=\log p_\theta(x).
    \end{align*}
    
    Now, we pay attention to the second term $\mathbb{E}_{p_\mathcal{D}(x)} \mathbb{E}_{q_\phi(z|x)} \left[ \log \frac{q_\phi(z)}{p(z|x)}\right]$.
    Then, we see that this term is equivalent to $\mathbb{KL}(q_\phi(x,z)\|p(x,z)) - \mathbb{I}_{q_\phi(z,x)}(x,y) $.
    \begin{align*}
         \mathbb{E}_{p_\mathcal{D}(x)} \mathbb{E}_{q_\phi(z|x)} \left[ \log \frac{q_\phi(z)}{p(z|x)}\right]
            &= \mathbb{E}_{p_\mathcal{D}(x)} \mathbb{E}_{q_\phi(z|x)} \left[ \log \frac{q_\phi(z)p(x)}{p(z,x)}\right]\\
            &= \mathbb{E}_{p_\mathcal{D}(x)} \mathbb{E}_{q_\phi(z|x)} \left[ \log \frac{q_\phi(z)p(x)}{q_\phi(z,x)}\frac{q_\phi(x,z)}{p(z,x)}\right]\\
            &= \mathbb{E}_{p_\mathcal{D}(x)} \mathbb{E}_{q_\phi(z|x)} \left[ \log \frac{q_\phi(z)p(x)}{q_\phi(z,x)}\right] + \mathbb{E}_{p_\mathcal{D}(x)} \mathbb{E}_{q_\phi(z|x)} \left[\log\frac{q_\phi(x,z)}{p(z,x)}\right]\\
            &= -\mathbb{I}_{q_\phi(x,z)}(x,z) + \mathbb{KL}(q_\phi(x,z)\|p(x,z))\\
    \end{align*}
\end{proof}

Finally, the relationship between $\log p_\theta(x)$ and $\log p_{\theta, \phi}(x)$ becomes: 

        \begin{corollary}
            Assume that $c(x,\tilde{x}) = \|x-\tilde{x}\|^2$,
            $p_\theta(\tilde{x}|z)=\mathcal{N}(\tilde{x};G_\theta(z), \sigma^2I)$
            for $W^{\ddagger}(p_\mathcal{D}(x), p_\theta)$. Then,
            \begin{align}
                \log p_\theta(x) \leq  \log p_{\theta, \phi}(x) - \mathbb{I}_{q_\phi(x,z)}\left[x,z\right]+\mathbb{KL}(q_\phi(x,z)\|p(x,z))
            \end{align}
        \end{corollary}
        The proof is simply applying Proposition~\ref{prop:wasser} 
        into Theorem~\ref{theorem1}.
        The gap between the autoencoder log-likelihood 
        and the generative model log-likelihood is 
        at least the difference between mutual information and the relative 
        information between $q_\phi(x,z)$ and $p(x,z)$.

\begin{theorem}
The difference between maximizing AAE versus VAE corresponds to
\begin{align} 
    \mathcal{L}_{\text{AAE}} - \mathcal{L}_{\text{VAE}} = 
        \mathbb{E}_{p_\mathcal{D}(x)} \left[\mathbb{KL}(q_\phi(z|x)\|q_\phi(z))\right]
\end{align} 
\end{theorem}
\begin{proof}
    \begin{align*}
        \mathcal{L}_{\text{AAE}} - \mathcal{L}_{\text{VAE}} &= 
            \log p(x) + \mathbb{I}_{q_\phi(x,z)}(x,z) -\mathbb{KL}(q_\phi(z,x)\|p(x,z)) - \log p(x) + \mathbb{KL}(q_\phi(z|x)\|p(z|x))\\
            &= \log p(x) + \mathbb{I}_{q_\phi(x,z)}(x,z) -\mathbb{KL}(q_\phi(z,x)\|p(x,z)) + \mathbb{KL}(q_\phi(z|x)\|p(z|x))\\
            &= \mathbb{E}_{p_\mathcal{D}(x)}\mathbb{E}_{q_\phi(z|x)} \left[\log \frac{q_\phi(x,z)}{q_\phi(z)p(x)}\right]\\
                &\qquad\qquad\qquad- \mathbb{E}_{p_\mathcal{D}(x)}\mathbb{E}_{q_\phi(z|x)}\left[\log\frac{p(x,z)}{q_\phi(x,z)}\right] 
                + \mathbb{E}_{p_\mathcal{D}(x)}\mathbb{E}_{q_\phi(z|x)}\left[\log \frac{q_\phi(z|x)}{p(z|x)}\right]\\
            &= \mathbb{E}_{p_\mathcal{D}(x)}\mathbb{E}_{q_\phi(z|x)} \left[\log 
                \left(\frac{q_\phi(x,z)}{q_\phi(z)p(x)}\right) 
                \left(\frac{p(x,z)}{q_\phi(x,z)}\right)
                \left(\frac{q_\phi(z|x)}{p(z|x)}\right)\right]\\
            &= \mathbb{E}_{p_\mathcal{D}(x)}\mathbb{E}_{q_\phi(z|x)} \left[\log \frac{q_\phi(z|x)}{q_\phi(z)}\right]\\
            &= \mathbb{E}_{p_\mathcal{D}(x)} \left[\mathbb{KL}(q_\phi(z|x)\|q_\phi(z))\right]
    \end{align*}
\end{proof}

\begin{proposition}
    For any distribution $p_{\mathcal{D}}(x)$ and $p_\theta(x)$:
    \begin{align*}
        D_{\text{IW-AAE}}(p_{D}, p_\theta) \leq D_{\text{AAE}}(p_{D}, p_\theta) \leq D_{\text{AVB}}(p_{D}, p_\theta)\\
        D_{\text{IW-AAE}}(p_{D}, p_\theta) \leq D_{\text{IW-AVB}}(p_{D}, p_\theta) \leq D_{\text{AVB}}(p_{D}, p_\theta)
    \end{align*}
\end{proposition}
\begin{proof}
    First, we show $D_{\text{IW-AAE}}(p_{D}, p_\theta) \leq D_{\text{AAE}}(p_{D}, p_\theta)$.
    \begin{align*}
        D_{\text{IW-AAE}}(p_{\mathcal{D}}, p_\theta) 
        &= \inf_{q_\phi\in\mathcal{Q}} D_{\text{GAN}}(q_\phi(z), p(z)) 
            - \mathbb{E}_{p_{\mathcal{D}}(x)}
            \mathbb{E}_{z_1 \ldots z_k\sim q_\phi(z|x)}\left[  \frac{1}{k}\sum^{k}_{i=1} \log p_\theta(x|z_i)\right]\\
            &\qquad \qquad  \qquad- \max_\theta 
                    \mathbb{E}_{p_{\mathcal{D}}(x)}
                    \mathbb{E}_{z_1,\ldots, z_k\sim q_\phi(z|x)}
                    \left[\log \frac{1}{k}\sum^k_{j=0} p_\theta (x|z_i)
                    \exp(-T(x,z_j))\right]\\
        &\leq \inf_{q_\phi\in\mathcal{Q}} D_{\text{GAN}}(q_\phi(z), p(z)) 
            - \mathbb{E}_{p_{\mathcal{D}}(x)}
            \mathbb{E}_{z_1 \ldots z_k\sim q_\phi(z|x)}\left[ \frac{1}{k}\sum^{k}_{i=1} \log p_\theta(x|z_i)\right] \\
            &\qquad \qquad  \qquad- \max_\theta 
                    \mathbb{E}_{p_{\mathcal{D}}(x)}
                    \mathbb{E}_{z_1,\ldots, z_k\sim q_\phi(z|x)}
                    \left[ \frac{1}{k}\sum^k_{j=0} \log p_\theta (x|z_i)\right] \\
            &= D_{\text{AAE}}.
    \end{align*}
    Proposition 4 in \cite{Bousquet2017} tells us that $D_{\text{AAE}}(p_{\mathcal{D}}, p_\theta) \leq D_{\text{AVB}}(p_{\mathcal{D}}, p_\theta)$.
    Hence, the first inequality holds true. 
    Now, we show $D_{\text{IW-AAE}}(p_{D}, p_\theta) \leq D_{\text{AIWB}}(p_{D}, p_\theta) \leq D_{\text{AVB}}(p_{D}, p_\theta)$ (**).
    \begin{align*}
        D_{\text{IW-AAE}}(p_{\mathcal{D}}, p_\theta) 
        &= \inf_{q_\phi\in\mathcal{Q}} D_{\text{GAN}}\left(\int_{\mathcal{X}} q_\phi(z|x)p_{\mathcal{D}}(x)dx, p(z)\right) 
            - \mathbb{E}_{p_{\mathcal{D}}(x)}\
            \mathbb{E}_{z_1 \ldots z_k\sim q_\phi(z|x)}\left[\frac{1}{k}\sum^{k}_{i=1} \log  p_\theta(x|z_i)\right]\\
            &\qquad \qquad  \qquad - \max_\theta 
                    \mathbb{E}_{p_{\mathcal{D}}(x)}
                    \mathbb{E}_{z_1,\ldots, z_k\sim q_\phi(z|x)}
                    \left[\log \frac{1}{k}\sum^k_{j=0} p_\theta (x|z_i) \exp(-T(x,z_j))\right]\\
        &\leq \inf_{q_\phi\in\mathcal{Q}} \mathbb{E}_{p_{\mathcal{D}}(x)} \left[ D_{\text{GAN}}(q_\phi(z|x), p(z)) 
            - \mathbb{E}_{z_1 \ldots z_k\sim q_\phi(z|x)}\left[ \frac{1}{k}\sum^{k}_{i=1} \log p_\theta(x|z_i)\right]\right]\\
            &\qquad \qquad  \qquad - \max_\theta 
                    \mathbb{E}_{p_{\mathcal{D}}(x)}
                    \mathbb{E}_{z_1,\ldots, z_k\sim q_\phi(z|x)}
                    \left[\log \frac{1}{k}\sum^k_{j=0} p_\theta (x|z_i) \exp(-T(x,z_j))\right]\\
                    &= D_{\text{IW-AVB}}.
    \end{align*}
    The last inequality is due to the joint convexity of $D_{\text{GAN}}$
    and Jensen's inequality.
    Hence, (**) inequality holds true. 
\end{proof}

\pagebreak

   \begin{algorithm}[H]
    \caption{Pseudocode for training IW-AAE}
    \label{algo:IW-AAE}
    \begin{algorithmic}[1]
        \State Initialize $\theta$, $\phi$, and $\psi$. 

        \While {$\phi$ has not converged}
            \State Sample $\lbrace \epsilon^{(1)}, \ldots, \epsilon^{(NK)} \rbrace \sim \mathcal{N}(0,1)$.
            \State Sample $\lbrace x^{(1)}, \ldots, x^{(N)} \rbrace \sim p_{\mathcal{D}}(x)$.
            \State Sample $\lbrace z^{(1)}, \ldots, z^{(NK)} \rbrace \sim p(z)$.
            \State Sample $\lbrace \tilde{z}^{(k|n)} \rbrace \sim q(z|x^{(n)}, \epsilon^(m))$.

            \State Compute gradient w.r.t $\theta$ in Eq.~\ref{eqn:IW-AAE_org} : 
            \State $\frac{1}{N}\sum^N_{n=1} \nabla_\theta \log \left[\frac{1}{K}\sum^{K}_{k=1} p_\theta \left(x^{(n)}|\tilde{z}^{(k|n)}\right)\right]$
            \State $\qquad\qquad\qquad \exp\left(- T_{\psi}\left(
            \tilde{z}^{(k|n)}\right)\right)\big]$
            \State Compute gradient w.r.t $\phi$ in Eq.~\ref{eqn:IW-AAE_org}.
            \State Compute gradient w.r.t $\psi$ in Eq.~\ref{eqnTstar_z}: 
            \State $\frac{1}{N}\sum^N_{n=1} \nabla_\psi \big[ 
                \log \left( \sigma(T_\psi(\tilde{z}^{(1|n)})))\right)$
            \State $\qquad \qquad + \log \left( 1 -  \sigma(T_\psi(z^{(n)}) \right) \big]$
        \EndWhile
    \end{algorithmic}
    \end{algorithm}

\subsection*{Wasserstein Autoencoder and AAE}
We provide brief background on Wasserestein autoencoder and relation to AAE.


    Let $c(x,\tilde{x}): \mathcal{X} \times \mathcal{X} \to \mathcal{R}_+$ be any measurable cost function and
        $p(x\sim p_{\mathcal{D}}, \tilde{x} \sim p_\theta)$ is a set of all 
        joint distributions over random variables $(x,\tilde{x})$ with marginals 
        $p_\mathcal{D}$ and $p_\theta$.
        The Kantorovich's formulation of optimal transport problem \cite{Kantorovich1942} is defined as 
        \begin{align}
            W_c(p_\mathcal{D}, p_\theta) := \inf_{\Gamma \in p(x\sim p_\mathcal{D}, \tilde{x}\sim p_\theta)}
                    \mathbb{E}_{(x,\tilde{x})\sim \Gamma}\left[ c(x,\tilde{x})\right].
        \end{align}
        Case when $c(x,\tilde{x}) = D^p(x,\tilde{x})$ with $(\mathcal{X}, D)$ 
        is a metric space is called {\em p-Waserstein distance}. 

        {\em Bousquet et al.} \cite{Bousquet2017} models the joint distribution class with latent generative modelling
        by formulating the joint distribution as $p(x,\tilde{x}) = \int p_\theta(\tilde{x}|z)q(z|x)p_{\mathcal{D}}(x)dz$\footnote{Note that $p(x,\tilde{x})$ is different from $p(x\sim p_\mathcal{D}, \tilde{x}\sim p_\theta)$
        where $p(x,\tilde{x})$ depends on $p_\theta(x,\tilde{x})$.}. Then, we have 
        \begin{align}
            W_c^\dagger(p_\mathcal{D}, p_\theta) := 
                \inf_{q_\phi: q_\phi(z)=p(z)}\mathbb{E}_{p_\mathcal{D}(x)}\mathbb{E}_{q_\phi(z|x)}\left[ c(x,G_\theta(z))\right]
        \end{align}
        where $\tilde{x}=G_\theta(z)$.
        Given that the generative network $p_\theta$ is probabilistic function, we have
        $W_c(p_\mathcal{D}, p_\theta) \leq W_c^\dagger(p_\mathcal{D}, p_\theta)$.

        In practice, the constraint $q_\phi(z) =p(z)$ has been relaxed by using convex penalty function,  
        \begin{align}
            D_{\text{POT}}(p_\mathcal{D}, p_\theta) := \inf_{q(z|x)\in\mathcal{Q}} \mathbb{E}_{p_\mathcal{D}(x)}\mathbb{E}_{q(z|x)} \left[ c(x,g(z))\right] + \lambda D_{\text{GAN}}(q(z),p(z)).
        \end{align}
        Here, GAN is used for the choice of convex divergence between the prior $p(z)$ and 
        the aggregated posterior $q(z)$ \cite{Bousquet2017}.
        As $\lambda \to \infty$, $D_{\text{POT}}$ converges to $W_c^\dagger(p_\mathcal{D}, p_\theta)$.  
        It turns out that $D_{\text{AAE}}$ is a special case of $D_{\text{POT}}$.
        This happens when the cost function $c$ is squared Euclidean distance and
        $P_G(\tilde{x}|z)$ is Gaussian $\mathcal{N}(\tilde{x}; G_\theta(z),\sigma^2 I)$.

We refer to \cite{Bousquet2017} for the detailed descriptions.

\subsection*{Semi-Supervised Learning Experiments} \label{semi_supervised_learning}
\begin{figure}[htp]
    \includegraphics[width=\linewidth]{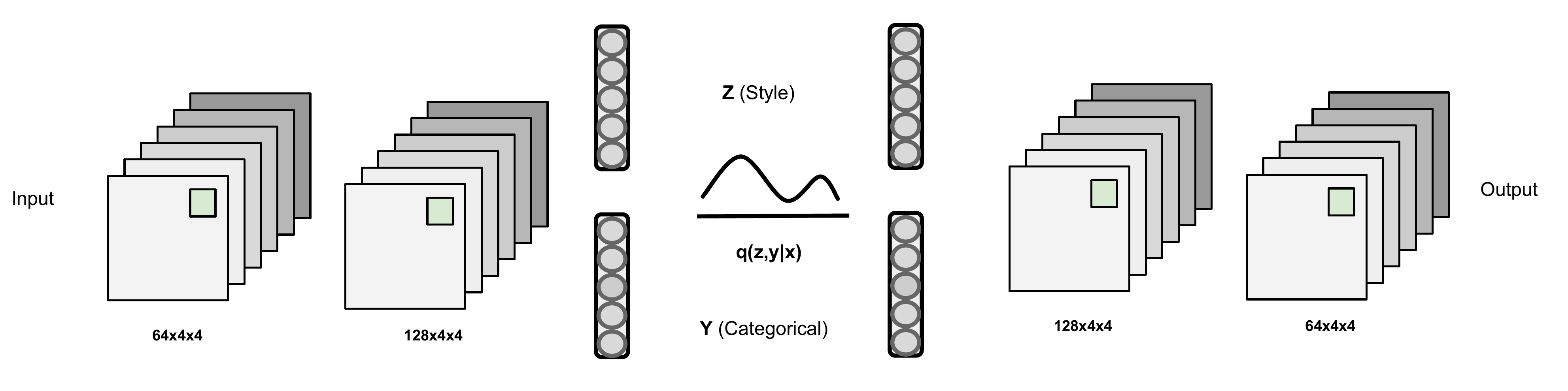}
    \caption{The encoder and decoder network architecture}
\end{figure}             

\begin{figure}[htp]
    \begin{minipage}{0.495\textwidth}
    \includegraphics[width=\linewidth]{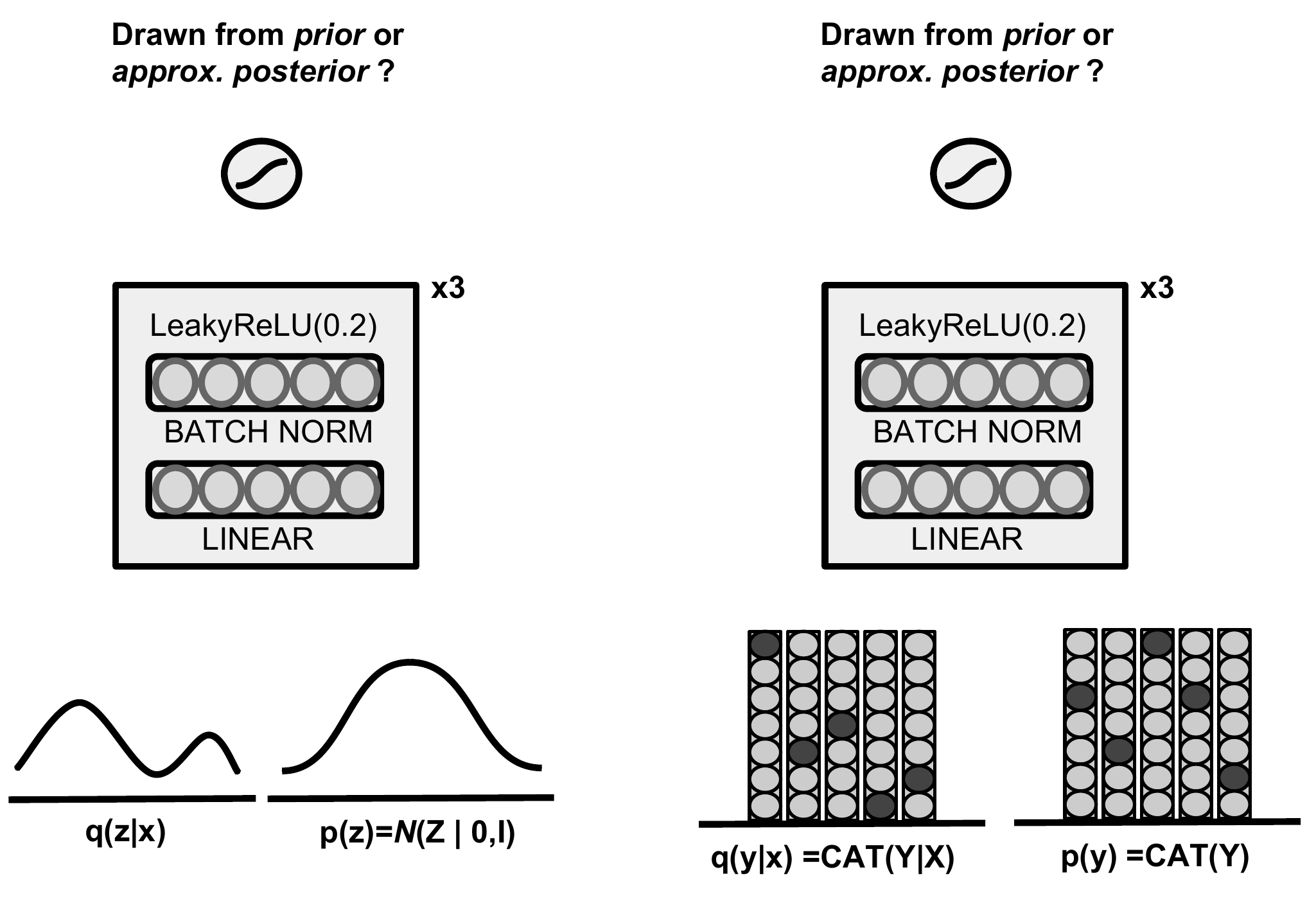}
    \subcaption{IW-AVB}
    \end{minipage}   
    \begin{minipage}{0.495\textwidth}
    \includegraphics[width=\linewidth]{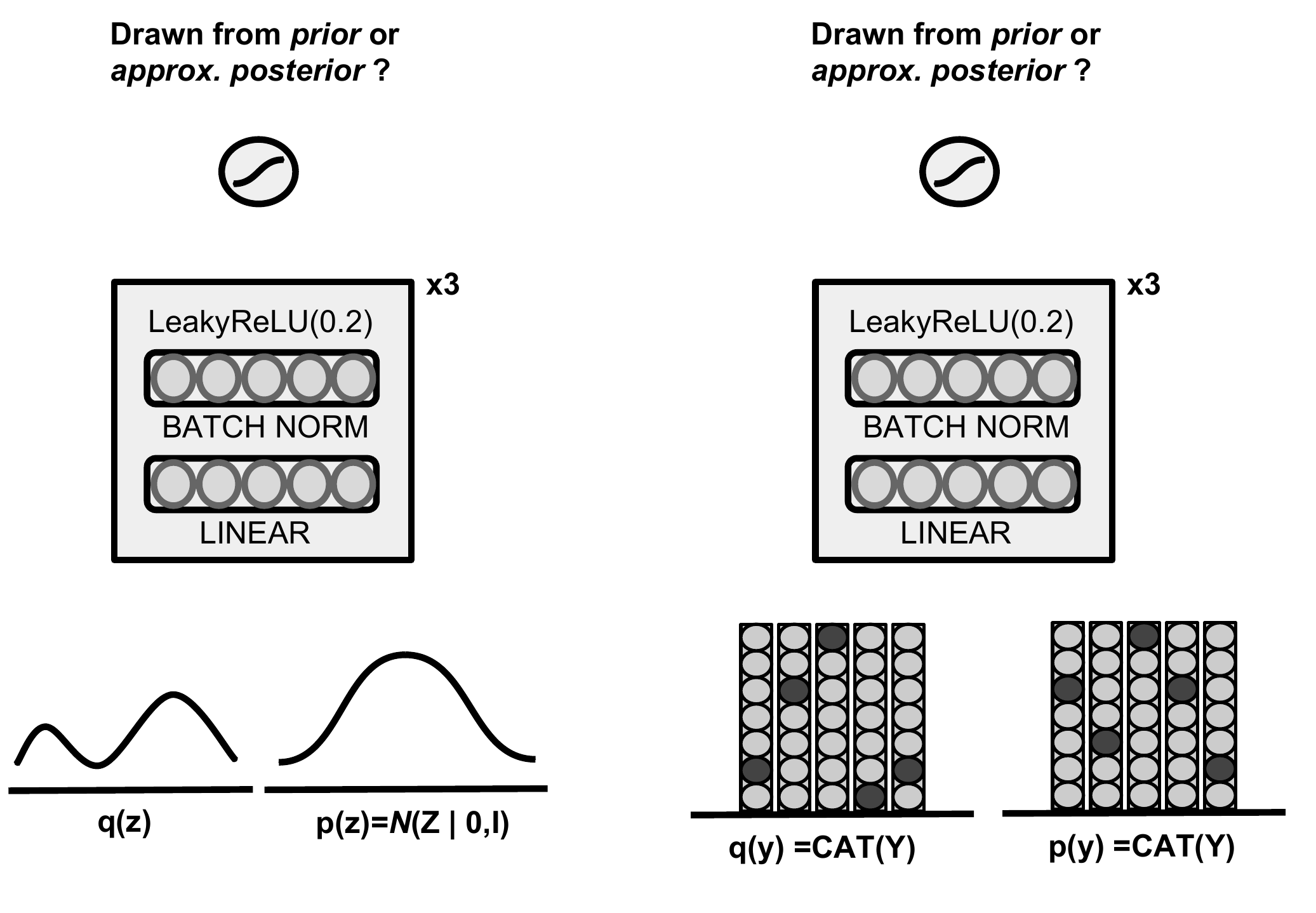}
    \subcaption{IW-AAE}
    \end{minipage}   
    \caption{The adversarial network architecture}
    \label{fig:semi_sup_arch}
    \vspace{-0.2cm}
\end{figure}             

Semi-supervised learning objective for IW-AVB is defined as following:
\begin{align}
    \log p(x) &\geq
       \mathbb{E}_{z_1,\cdots, z_k\sim q_\phi(\bf{z}|\bf{x})}
            \left[ \log \frac{1}{k}\sum^{k}_{i=1} p(x|z_i,y_i)  
            \exp (- T^*(x,z_i,y_i))  \right]\\
    T^*(x,z,y) &= \log q_\phi(z,y|x) - \log p(z,y) \nonumber \\
    &= \max_{T}\left[ 
        \mathbb{E}_{p_{\mathcal{D}}(x)} 
        \left[\mathbb{E}_{q_\phi(z,y|x)}
        \log \sigma (T(x,z,y)) + \mathbb{E}_{p(z,y)}
        \log \left(1- \sigma(T(x,z,y))\right)\right] \right] 
\end{align}

Assume that $p(z,y)=p(z)p(y)$ and $q(z,y|x)=q(z|x)q(y|x)$.
Then, we have 
\begin{multline}
    \log p(x) \geq
       \mathbb{E}_{z_1,\cdots, z_k\sim q_\phi(\bf{z}|\bf{x})}
            \bigg[ \log \frac{1}{k}\sum^{k}_{i=1} p_\theta(x|z_i,y_i)  
            \exp (-T1^*(x,z_i)) 
            \exp (-T2^*(x,y_i)) 
            \bigg]\label{eqn:ss_aiwbo}
\end{multline}
\begin{align}
    T1^*(x,z) &= \log q_\phi(z|x) - \log p(z) \nonumber \\
             &= \max_{T}\left[ 
        \mathbb{E}_{p_{\mathcal{D}}(x)} 
        \left[\mathbb{E}_{q_\phi(z|x)}
        \log \sigma (T1(x,z)) + \mathbb{E}_{p(z)}
        \log \left(1- \sigma(T1(x,z))\right)\right] \right]
        \label{eqn:ss_tstar_xz}\\
    T2^*(x,y) &= \log q_\psi(y|x) - \log p(y) \nonumber \\
             &= \max_{T}\left[ 
        \mathbb{E}_{p_{\mathcal{D}}(x)} 
        \left[\mathbb{E}_{q_\psi(y|x)}
        \log \sigma (T2(x,y)) + \mathbb{E}_{p(y)}
        \log \left(1- \sigma(T2(x,y))\right)\right] \right] 
        \label{eqn:ss_tstar_xy}
\end{align}
The factorization of prior and approximate posterior distributions $p(z,y)=p(z)p(y)$ and $q(z,y|x)=q(z|x)q(y|x)$
were design choice. However, note that it does not need to be factorized such a way. 
Furthermore, we take the same approach for IW-AAE.

\begin{algorithm}[htp]
    \caption{Semi-Supervised IW-AVB Training for MNIST}
    \label{algo:semi_IW-AVB}
    \begin{algorithmic}[1]
        \State Initialize $\theta$, $\psi$, $\phi$, and $\varphi$. 

        \While {$\phi$ has not converged}
            \State Sample $\lbrace \epsilon^{(1)}, \ldots, \epsilon^{(N)} \rbrace \sim \mathcal{N}(0,1)$.
            \State Sample $\lbrace x^{(1)}, \ldots, x^{(N)} \rbrace \sim p_{\mathcal{D}}(x)$.
            \State Sample $\lbrace z^{(1)}, \ldots, z^{(NK)} \rbrace \sim p(z)$.
            \State Sample $\lbrace y^{(1)}, \ldots, y^{(NK)} \rbrace \sim p(y)$.
            \State Sample $\lbrace \tilde{z}^{(k|n)} \rbrace \sim q_\varphi(z|x^{(n)},\epsilon^{m})$.
            \State Sample $\lbrace \tilde{y}^{(k|n)} \rbrace \sim q_\psi(y|x^{(n)})$.

            \State Compute gradient w.r.t $\theta$ in Eq.~\ref{eqn:ss_aiwbo}: 
            \State $\frac{1}{N}\sum^N_{n=1} \nabla_\theta \log \left[\frac{1}{K}\sum^{K}_{k=1} p_\theta \left(x^{(n)}|\tilde{z}^{(k|n)}, \tilde{y}^{(k|n)}\right)\exp\left(-T1_\varphi(x,\tilde{z}^{(k|n)})-T2_\psi(x,\tilde{y}^{(k|n)})\right)\right]$
            \State Compute gradient w.r.t $\phi$ in Eq.~\ref{eqn:ss_aiwbo}: 
            \State $\frac{1}{N}\sum^N_{n=1}\frac{1}{K}\sum^K_{k=1} \nabla_\phi \big[\log p_\theta \left(x^{(n)}|\tilde{z}^{(1|n)}, \tilde{y}^{(1|n)}\right) - T1_{\varphi}\left(x^{(n)}, \tilde{z}^{(1|n)}\right) - T2_{\psi}\left(x^{(n)}, \tilde{y}^{(1|n)}\right)\big]$
            \State Compute gradient w.r.t $\varphi$ in Eq.~\ref{eqn:ss_tstar_xz}: 
            \State $\frac{1}{N}\sum^N_{n=1} \frac{1}{K}\sum^K_{k=1} \nabla_\varphi \big[ 
                \log \left( \sigma(T1_\varphi(x^{(n)}, \tilde{z}^{(1|n)}))\right) + \log \left( 1 -  \sigma(T1_\varphi(x^{(n)}, z^{(n)})) \right) \big]$
            \State Compute gradient w.r.t $\psi$ in Eq.~\ref{eqn:ss_tstar_xy}: 
            \State $\frac{1}{N}\sum^N_{n=1} \nabla_\psi \big[ 
                \log \left( \sigma(T2_\psi(x^{(n)}, \tilde{y}^{(1|n)}))\right) + \log \left( 1 -  \sigma(T2_\psi(x^{(n)}, y^{(n)})) \right) \big]$
            \State Compute gradient w.r.t $\psi$ on Cross-Entropy Loss $\mathcal{L}(x,t)$ where $t$ is ground truth label.
        \EndWhile
    \end{algorithmic}
\end{algorithm}

\begin{algorithm}[htp]
    \caption{Semi-Supervised IW-AVB Training for Modeling Neuron Spikes}
    \label{algo:semi_IW-AVB_neuron}
    \begin{algorithmic}[1]
        \State Initialize $\theta$, $\psi$, and $\phi$. 

        \While {$\phi$ has not converged}
            \State Sample $\lbrace \epsilon^{(1)}, \ldots, \epsilon^{(N)} \rbrace \sim \mathcal{N}(0,1)$.
            \State Sample $\lbrace x^{(1)}, \ldots, x^{(N)} \rbrace \sim p_{\mathcal{D}}(x)$.
            \State Sample $\lbrace y^{(1)}, \ldots, y^{(NK)} \rbrace \sim p(y)$.
            \State Sample $\lbrace \tilde{y}^{(k|n)} \rbrace \sim q_\psi(y|x^{(n)})$.

            \State Compute gradient w.r.t $\theta$ in Eq.~\ref{eqn:ss_aiwbo}: 
            \State $\frac{1}{N}\sum^N_{n=1} \nabla_\theta \log \left[\frac{1}{K}\sum^{K}_{k=1} p_\theta \left(x^{(n)}|\tilde{z}^{(k|n)}, \tilde{y}^{(k|n)}\right)\exp\left(-T_\psi(x,\tilde{y}^{(k|n)})\right)\right]$
            \State Compute gradient w.r.t $\phi$ in Eq.~\ref{eqn:ss_aiwbo}: 
            \State $\frac{1}{N}\sum^N_{n=1}\frac{1}{K}\sum^K_{k=1} \nabla_\phi \big[\log p_\theta \left(x^{(n)}|\tilde{z}^{(1|n)}, \tilde{y}^{(1|n)}\right) - T_{\psi}\left(x^{(n)}, \tilde{y}^{(1|n)}\right)\big]$
            \State Compute gradient w.r.t $\psi$ in Eq.~\ref{eqn:ss_tstar_xy}: 
            \State $\frac{1}{N}\sum^N_{n=1} \nabla_\psi \big[ 
                \log \left( \sigma(T2_\psi(x^{(n)}, \tilde{y}^{(1|n)}))\right) + \log \left( 1 -  \sigma(T_\psi(x^{(n)}, y^{(n)})) \right) \big]$
            \State Compute gradient w.r.t $\psi$ on Cross-Entropy Loss $\mathcal{L}(x,t)$ where $t$ is ground truth label.
        \EndWhile
    \end{algorithmic}
\end{algorithm}

\subsection{Experiments}
 We follow the same experimental procedure as 
\cite{Mescheder2017} for learning generative models on binarized MNIST dataset.
We trained AVB, IW-AVB, AAE, and IW-AAE on 50,000 train examples with 10,000 validation examples,
and measured log-likelihood on 10,000 test examples.
We used 5-layer deep convolutional neural network for the 
        decoder network and we used fully connected 4-layer neural 
        network with 1024 units in each hidden layer for the adversary 
        network. For encoder network is defined such that we can efficiently computes 
        the moments of $q_\phi(z|x)$ by linearly combining learned noise 
        basis vector with the coefficients depend on the data points,
        $z_k = \sum^m_{i=1} v_{i,k}(\epsilon_i)a_{i,k}(x)$
        where $z_k$ are the $k^{th}$ latent variable, $v_{i,k}(\epsilon_i)$
        is the learned noised basis are parameterized with full-connected
        neural network, and $a_{i,k}(x)$ are the coefficient that is
        parameterized using the deep convolutional neural network.
        We set our latent dimension $z$ to be $32$.
        We used adaptive contrast method during the training for all models,
        since it has been shown that adaptive contrast method gives superior performance in \cite{Mescheder2017}

Here is the samples generated on MNIST and FashionMNIST dataset.
\begin{figure*}[t]
    \centering
    \begin{minipage}{0.245\textwidth}
        \includegraphics[width=0.99\linewidth]{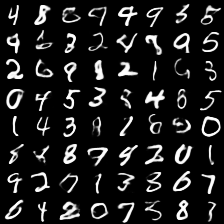}
        \subcaption{AVB}
    \end{minipage}
    \begin{minipage}{0.245\textwidth}
        \includegraphics[width=0.99\linewidth]{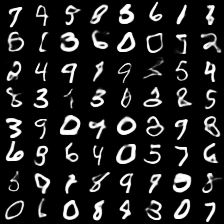}
        \subcaption{IW-AVB}
    \end{minipage}
    \begin{minipage}{0.245\textwidth}
        \includegraphics[width=0.99\linewidth]{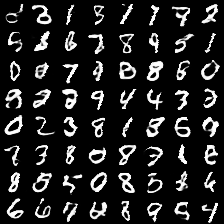}
        \subcaption{AAE}
    \end{minipage}
    \begin{minipage}{0.245\textwidth}
        \includegraphics[width=0.99\linewidth]{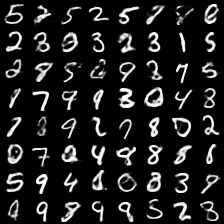}
        \subcaption{IW-AAE}
    \end{minipage}\\
    \caption{Samples of generative models from training MNIST dataset. }
\end{figure*}
\begin{figure*}[t]
    \centering
    \begin{minipage}{0.245\textwidth}
        \includegraphics[width=0.99\linewidth]{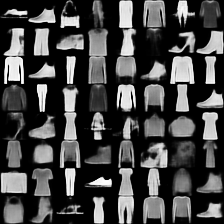}
        \subcaption{AVB}
    \end{minipage}
    \begin{minipage}{0.245\textwidth}
        \includegraphics[width=0.99\linewidth]{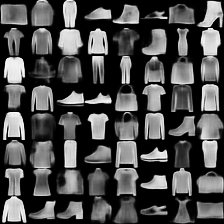}
        \subcaption{IW-AVB}
    \end{minipage}
    \begin{minipage}{0.245\textwidth}
        \includegraphics[width=0.99\linewidth]{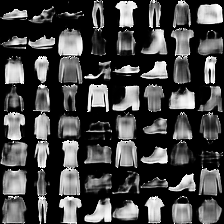}
        \subcaption{AAE}
    \end{minipage}
    \begin{minipage}{0.245\textwidth}
        \includegraphics[width=0.99\linewidth]{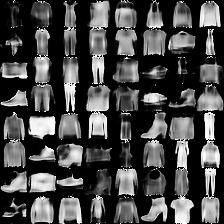}
        \subcaption{IW-AAE}
    \end{minipage}\\
    \caption{Samples of generative models from training FashionMNIST dataset. }
    \label{fig:fmnist_samples}
\end{figure*}

\subsection*{Semi-supervised learning}
     Next, we considered testing IW-AVB and IW-AAE on semi-supervised setting. 
        We used 100 and 1000 labels of MNIST and FahsionMNIST training data.
        We followed the same experimental setup that was used for
        testing semi-supervised learning in \cite{Makhzani2016}.

        In semi-supervised settings, we define the IW-AVB objective as follows: 
        \begin{align}
            \log p(x) \geq
            \mathbb{E}_{z_1,\cdots, z_k\sim q_\phi(\bf{z}|\bf{x})}
                \bigg[ \log \frac{1}{k}\sum^{k}_{i=1} p_\theta(x|z_i,y_i)  
                \exp \left(-T1^*(x,z_i)-T2^*(x,y_i)\right) 
            \bigg]
        \end{align}    
\begin{align}
    T_1^*(x,z) &= \log q_\phi(z|x) - \log p(z) \nonumber \\
    \begin{split}
             = \max_{T_1}\big[ 
        \mathbb{E}_{p_{\mathcal{D}}(x)} 
        \big[\mathbb{E}_{q_\phi(z|x)}
        \log \sigma (T_1(x,z)) \\ + \mathbb{E}_{p(z)}
        \log \left(1- \sigma(T_1(x,z))\right)\big] \big]
    \end{split}
\end{align}
\begin{align}
    T_2^*(x,y) &= \log q_\psi(y|x) - \log p(y)\nonumber\\
    \begin{split}
        = \max_{T_2}\big[ 
        \mathbb{E}_{p_{\mathcal{D}}(x)} 
        \big[\mathbb{E}_{q_\psi(y|x)}
        \log \sigma (T_2(x,y))\\ + \mathbb{E}_{p(y)}
        \log \left(1- \sigma(T_2(x,y))\right)\big] \big] 
    \end{split}
\end{align}
        where we have two adversarial networks $T_1$ and $T_2$.
        The first adversarial network distinguishes between style samples $z$ from 
        $p(z)$ versus $q_\phi(z|x)$, and the second adversarial network 
        distinguishes between label samples $y$ from $p(y)$ versus $q_\psi(y|x)$.
        We take the same approach to define IW-AAE objective as well.
        The depiction of semi-supervised learning framework for IW-AVB and IW-AAE,
        and the derivation of objectives are shown in Supplementary Materials Section~\ref{semi_supervised_learning}.
        
        For our datasets, we assume the data is generated by two types of
        latent variables $z$ and $y$ that comes from Gaussian and Categorical
        distribution, $p(y)=\text{Cat}(y)$ and $p(z)=\mathcal{N}(z|0,I)$.
        The encoder network $q(z,y|x)=q_\phi(z|x)q_\psi(y|x)$ outputs both standard latent variable $z$
        which is responsible for style representation, 
        and one-hot variable $y$ which is responsible for label representation.
        We impose a Categorical distribution for label representation.
        
        For encoder network, we used 2-layer convolutional neural network, followed by 
        fully connected hidden layer that outputs latent variable $z$ and $y$.
        For decoder network, we used single layer fully connected layer that takes $z$ and $y$
        and propagate them to 2-layer convolutional neural network that
        generates the samples. 
        For adversarial network, we used 4-layer fully connected neural 
        network that takes $z$ as inputs for IW-AAE and takes $x$ and 
        $z$ as inputs for IW-AVB.
        The architecture for AVB and IW-AVB has been shared thorough out 
        the experiments, and similarly for AAE and IW-AAE.
        See Figure~\ref{fig:semi_sup_arch} for architecture layout and 
        Algorithm~\ref{algo:semi_IW-AVB} for pseudo-code in the supplementary materials.

\begin{wraptable}{l}{0.5\textwidth}
  \begin{minipage}{0.5\textwidth}
  \centering
  \caption{Semi-Supervised Learning Performance on MNIST}
  \label{tab:semi_supervised_results}
  \begin{tabular}{|l|cc|cc|}\hline
                            & \multicolumn{2}{c}{MNIST} & \multicolumn{2}{|c|}{FMNIST} \\\hline 
                            & 100           & 1000          & 100   & 1000 \\\hline\hline 
      NN                    & 70.55         & 90.16         & 71.34 & 80.13 \\
      AVB                   & 87.14         & 96.70         & 77.08 & 83.49\\
      IW-AVB                & 87.99        & {\bf 97.33}   & 77.03 & {\bf 83.94}\\
      AAE                   & 88.13         & 96.73         & {\bf 77.68} & 83.41 \\        
      IW-AAE                & {\bf 91.08}  & 97.11         & 77.67 & 83.93 \\      \hline
  \end{tabular}
  \end{minipage}
\end{wraptable}

        We measure the accuracy of the classifier $q_\psi(y|x)$ to measure the 
        semi-supervised learning performance.
        We also include supervised trained two-layer fully connected neural 
        network (NN) with ReLU units as a baseline. 
        The results are shown in Table~\ref{tab:semi_supervised_results}.
        We observe that IW-AAE performs the best and followed by IW-AVB for
        100 label settings and IW-AVB performs the best and followed by IW-AAE
        for 100 label settings.




\subsection*{Neural Spike Modeling}\label{spike_inf}

{\bf Model Architecture}
During the experiments, we used {\em biophyiscal model} as a generative models for spike-inference experiments. In particular, we use is a linear model where the calcium process is modeled as an exponential decay, and the fluorescent measurement is a scaled readout of the calcium process, 
\begin{eqnarray}
    \frac{dc}{dt} = -\frac{1}{\tau} c+ s(t) \label{eqn:calcium_dyn}\\ 
    f(t) = \alpha c(t) + \beta + \epsilon(t)
\end{eqnarray}
where $s(t) = \sum_{i} \delta (t - t_i)$ is the spike train, $\tau$ is the decay constant, $\alpha$ is the amplitude of the fluorescence for one spike, and $\beta$ is the baseline. $\epsilon \sim \mbox{Normal}(0, \sigma)$ is noise drawn from Gaussian distribution. Thus our generative parameters are $\theta=\{\tau, \alpha, \beta, \sigma\}$. We discretize the dynamic equation (\ref{eqn:calcium_dyn}) using Euler method with discretization step $~0.16 \mbox{ms}$ (60Hz). 

The four layer convolutional neural network was used as encoder with ReLU activations. The filter size of convolutional neural network were 31, 21, 21, 11 and number of filters on each layers were 20, 20, 20, 20 respectively. For AVB, AAE, IW-AAE, and IW-AVB, we injected extra noise channel to the input on first two layers of convolutional layer. 

{\bf Baseline Models}
VIMCO-F and VIMCO-NF are the baseline models that we use for the spike inference problem. Both the model use a multi-layer convolutional neural network as the encoder model (4 hidden layers, each layer has 20 one-dimensional filters. The sizes are 31, 31, 21, 21 for first hidden layers to last hidden layers.). The VIMCO-F model uses factorized posterior distribution (independent samples are drawn from the Bernoulli distribution dictated by the output of convolutional neural network. On the other hand, the VIMCO-NF model has an autoregressive layer that add correlations among samples, hence non-factorized posterior distribution. The autoregressive layer draw a new sample at time $t$ conditioned on the past ten samples $t-1, \dots, t-10$. Both the generative model and the inference network is trained unsupervisedly using variational inference for monte carlo objectives (VIMCO) approach~\cite{mnih2016variational}.

\begin{figure*}[t]
    \centering
    \includegraphics[width=0.99\linewidth]{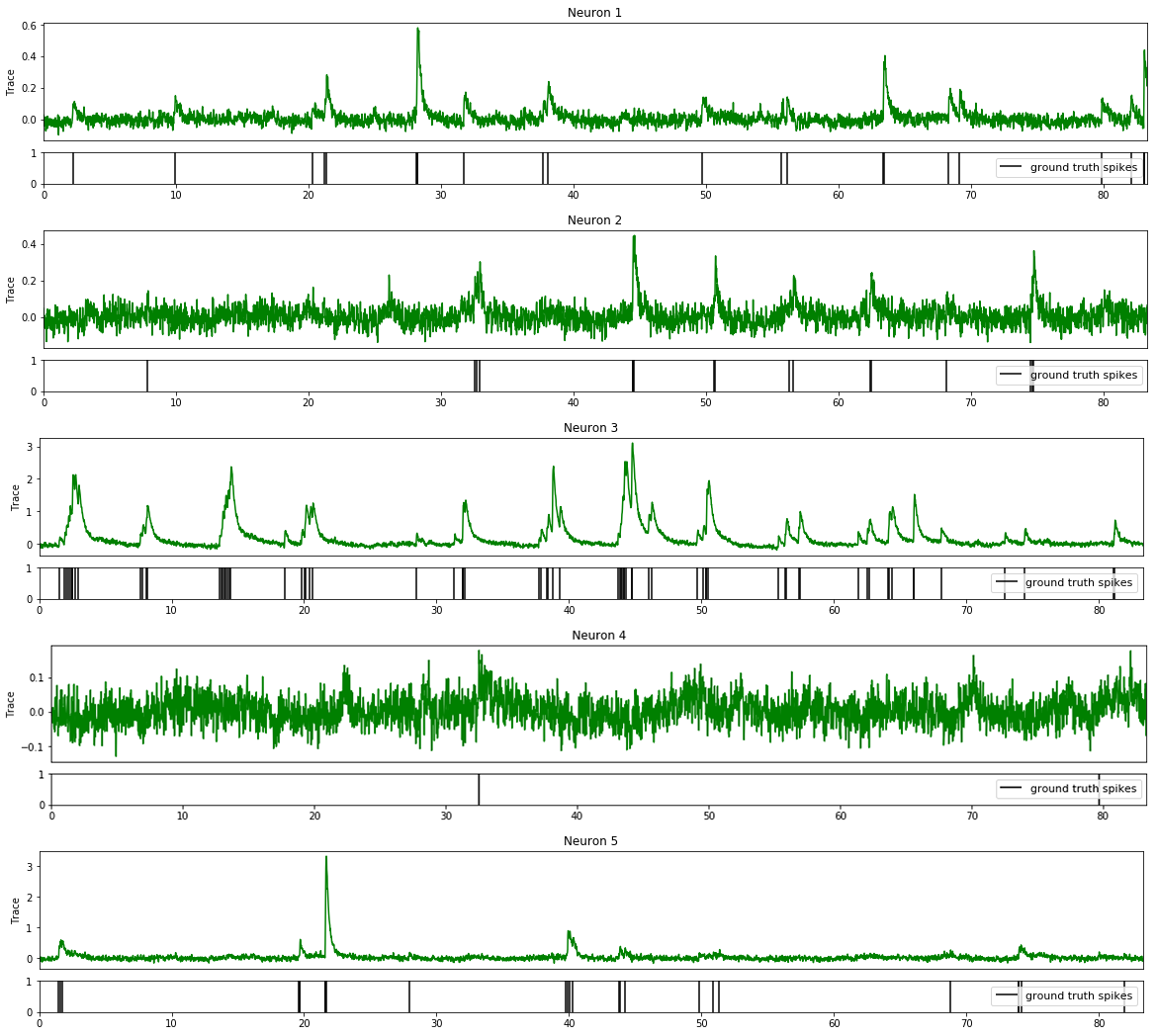}
    \caption{Neuron Fluorescence Traces and Spikes}
    \label{fig:neuron_trace}
\end{figure*}

\end{document}